\documentclass{elsarticle}
\usepackage[utf8]{inputenc}
\usepackage{csquotes}
\usepackage{soul}
\setuldepth{article}

\usepackage{hyperref}

\usepackage{amssymb}
\usepackage{amsmath}
\usepackage{amsthm} 
\usepackage{mathptmx}
\usepackage{centernot}
\usepackage{multirow}
\usepackage{lineno}

\usepackage{booktabs}

\usepackage{float}
\usepackage{subfig}

\makeatletter
\renewcommand{\p@subfigure}{}
\makeatother

\usepackage[inline, shortlabels]{enumitem}
 
\usepackage[shortcuts, theorems, arg]{macros}

\usepackage{dcolumn}

\newcolumntype{d}[1]{D{.}{.}{#1}}

\begin{document}

\begin{frontmatter} 
\title{Set Contribution Functions for Quantitative Bipolar Argumentation and their Principles} %
\author{Filip Naudot}
\ead{filipn@cs.umu.se}
\author{Andreas Brännström}
\ead{andreasb@cs.umu.se}
\author{Vicenç Torra}
\ead{vtorra@cs.umu.se}
\author{Timotheus Kampik}
\ead{tkampik@cs.umu.se}
\address{Umeå University, Sweden}

\begin{abstract}
We present functions that quantify the contribution of a set of arguments in quantitative bipolar argumentation graphs to (the final strength of) an argument of interest---a so-called \emph{topic}.
Our \emph{set contribution functions} are generalizations of existing functions that quantify the contribution of a single contributing argument to a topic.
Accordingly, we generalize existing contribution function principles for set contribution functions and provide a corresponding principle-based analysis.
We introduce new principles specific to set-based functions that focus on properties pertaining to the interaction of arguments within a set.
Finally, we sketch how the principles play out across different set contribution functions given a recommendation system application scenario.
\end{abstract}

\end{frontmatter}
\section{Introduction}
\label{sec:intro}
Computational Argumentation (CA) refers to a collection of methods for dialectical reasoning about potentially conflicting knowledge, often using graph-based approaches.
In graph-based CA, arguments are typically modeled as nodes in a graph and the relationships between arguments are modeled by one or several binary relations on the arguments, typically representing \emph{attack} or \emph{support} between nodes.
In recent years, the study of CA has received increasing attention due to its potential to provide a somewhat intuitive bridge between human and machine reasoning, as well as between symbolic and subsymbolic inference~\cite{DBLP:journals/corr/abs-2405-10729,DBLP:journals/frai/DietzKM22}.
A frequently studied CA approach is quantitative bipolar argumentation, in which arguments with numeric weights---so called initial strengths---are related via support and attack relations.
The graphs are called \emph{Quantitative Bipolar Argumentation Graphs} (QBAGs).
Initial strengths and graph topology are used by so-called \emph{gradual semantics} for drawing inferences by updating the initial strengths of the arguments to \emph{final} strengths~\cite{Baroni:Rago:Toni:2019,AMGOUD201839}.
Several applications and fundamental methods that draw inferences using QBAGs rely on meta-reasoning functions to quantify how much one argument (called \emph{contributor}) contributes to the final strengths of another argument (called \emph{topic})~\cite{yin2024qarg, DBLP:conf/ecai/0007PT23, DBLP:conf/ifaamas/AnaissyDVY25}.
Accordingly, such functions are sometimes called \emph{contribution functions}.

Scenarios for QBAG contribution functions are, for example, the following:
\begin{description}
    \item[Scenario 1.] Quantify how much a marginal change to an initial strength of one of several contributors affects the final strength of a topic to identify somewhat \emph{small} changes that, when applied, achieve a change of the topic's final strength in the desired direction.
    \item[Scenario 2.] Quantify how much the presence of a contributor affects the final strength of a topic.
    \item[Scenario 3.] Quantify how contributions to the final strength of a topic are distributed among several contributors.
\end{description}
Note that often, the focus of contribution function principles is on \emph{acyclic} QBAGs, as this restriction simplifies practical computation, as well as theoretical analysis, and is pragmatic in many use-cases (e.g.,~\cite{kotonya2019gradual,cocarascu2019extracting,lidecision,chi2021optimized,Potyka0T23,AyoobiPT23}).
Accordingly, our analysis focuses on the acyclic case as well.

Clearly, the three scenarios sketched above require different contribution functions.
Accordingly, several contribution functions have been defined in the literature, and contribution function \emph{principles} have been established that allow to distinguish between contribution functions in a rigorous manner~\cite{DBLP:journals/ijar/KampikPYCT24,Cyras:Kampik:Weng:2022,DBLP:conf/ecai/0007PT23,DBLP:conf/ifaamas/AnaissyDVY25}.

However, existing contribution functions cover the quantification of the contribution of a \emph{single} argument to a topic, which means in the above scenarios, they leave some questions unanswered.
\begin{itemize}
    \item In Scenario 1, one typically wants to find the most impactful contribution that a marginal change of any of several contributors can achieve, to then adjust the initial strength of this most impactful contributor accordingly.
    \item In Scenario 2, one may as well be interested in the effect the presence of several contributors has on a topic.
    \item In Scenario 3, one may want to group several contributors, for example all arguments that have been advanced by a specific agent, and treat each such group as a single player when assessing how contributions to a topic's final strength are distributed, allowing quantification of the group's overall impact relative to others.
\end{itemize}
To address these issues, this paper introduces and analyzes \emph{set contribution functions} that can quantify the contribution of a set of arguments to a topic (Section~\ref{sec:sctrb-functions}).
Our new functions generalize (single argument) contribution functions, in that the latter are covered by a special case (set contributor with cardinality of one) of the former (Section~\ref{sec:generalization}).
With this, we address Scenario 1 by introducing an abstract \emph{gradient} set contribution function $\sctrbgempty$ that aggregates several single-argument gradient contributions; from this, we then instantiate a function $\sctrbgmempty$ that determines the maximal gradient and thus most effective positive change.
We address Scenario 2 by introducing the \emph{removal-based} set contribution function $\sctrbrempty$ that determines the difference between a topic's initial strength with and without the set contributor.
A second variant---the \emph{intrinsic} removal-based set contribution function $\sctrbriempty$---controls for indirect effects on the set contributors.
Finally, a Shapley-value based set contribution function $\sctrbsempty$, which can be further generalized to determine contributions across an arbitrary \emph{partition} of the set of arguments, addresses Scenario 3.
To clarify the impact of this generalization, let us provide an illustrative example.
\begin{example}\label{ex:intro}
    Consider the QBAG displayed in Subfigure~\ref{fig:proof-remove-consistency}, with topic argument $\arga$ and contributing arguments $\argd$ and $\argf$. When assessed individually (using QE semantics $\fs$), both $\argd$ and $\argf$ have a negative removal-based ($\sctrbrempty$) effect on the final strength of $\arga$. However, when the two are considered together as a \emph{set contributor} $\{\argd, \argf\}$, their joint contribution to $\arga$ becomes positive. This sign inconsistency is illustrated in Subfigure~\ref{fig:sign-map}, and shows that set contributions can exhibit effects that would be difficult to measure using single-argument contributions alone. In other words, while single-argument contribution functions allow us to understand how each argument individually affects the topic, only set contribution functions can reveal how combinations of arguments interact in a way that fundamentally alters their overall impact.
\end{example}
\begin{figure}[H]
    \centering
    \subfloat[A QBAG $\graph$ (Example \ref{ex:intro}). A node labelled {\scriptsize $\argnode{\argx}{\is(\argx)}{\fs(\argx)}$} represents argument $\argx$ with initial strength $\is(\argx)$ and final strength $\fs(\argx)$. Edges labelled $+$ and $-$, respectively, represent support and attack relations.]{
    \centering
        \begin{tikzpicture}[node distance=1.5cm]
            \node[unanode] (d) at (-2, 2) {\argnode{\argd}{0.55}{0.55}};
            \node[unanode] (f) at (2, 2) {\argnode{\argf}{0.6}{0.60}};
            \node[unanode] (e) at (2, 0) {\argnode{\arge}{0.45}{0.57}};
            \node[unanode] (c) at (0, 0) {\argnode{\argc}{0.1}{0.61}};
            \node[unanode] (b) at (-2, 0) {\argnode{\argb}{0.8}{0.95}};
            \node[unanode] (a) at (0, -2) {\argnode{\arga}{0.3}{0.39}};
            
            \draw[-stealth, thick] (d) -- node[pos=0.5, left=1pt] {+} (c);
            \draw[-stealth, thick] (d) -- node[pos=0.5, left=1pt] {+} (e);
            \draw[-stealth, thick] (d) -- node[pos=0.5, left=1pt] {+} (b);
    
            \draw[-stealth, thick] (f) -- node[pos=0.5, right=1pt] {+} (c);
            \draw[-stealth, thick] (f) -- node[pos=0.5, right=1pt] {+} (e);
            \draw[-stealth, thick] (f) -- node[pos=0.5, right=1pt] {+} (b);
    
            \draw[-stealth, thick] (c) -- node[pos=0.5, below=1pt] {-} (e);
            \draw[-stealth, thick] (c) -- node[pos=0.5, below=1pt] {+} (b);
            
            \draw[-stealth, thick] (e) -- node[pos=0.5, right=1pt] {-} (a);
            
            \draw[-stealth, thick] (b) -- node[pos=0.5, left=1pt] {+} (a);

        \end{tikzpicture}
    \label{fig:proof-remove-consistency}
    }
    \hspace{5pt}
    \centering
    \subfloat[Sign map for removal-based set contributions of \{d\}, \{f\}, and \{d,f\} to the topic argument $\arga$. The darker region corresponds to cases where $\sctrbr{\{d\}}{a}$ and $\sctrbr{\{f\}}{a}$ have the same sign, whereas $\sctrbr{\{d, f\}}{a}$ has the opposite sign.]{
        \includegraphics[width=0.5\linewidth]{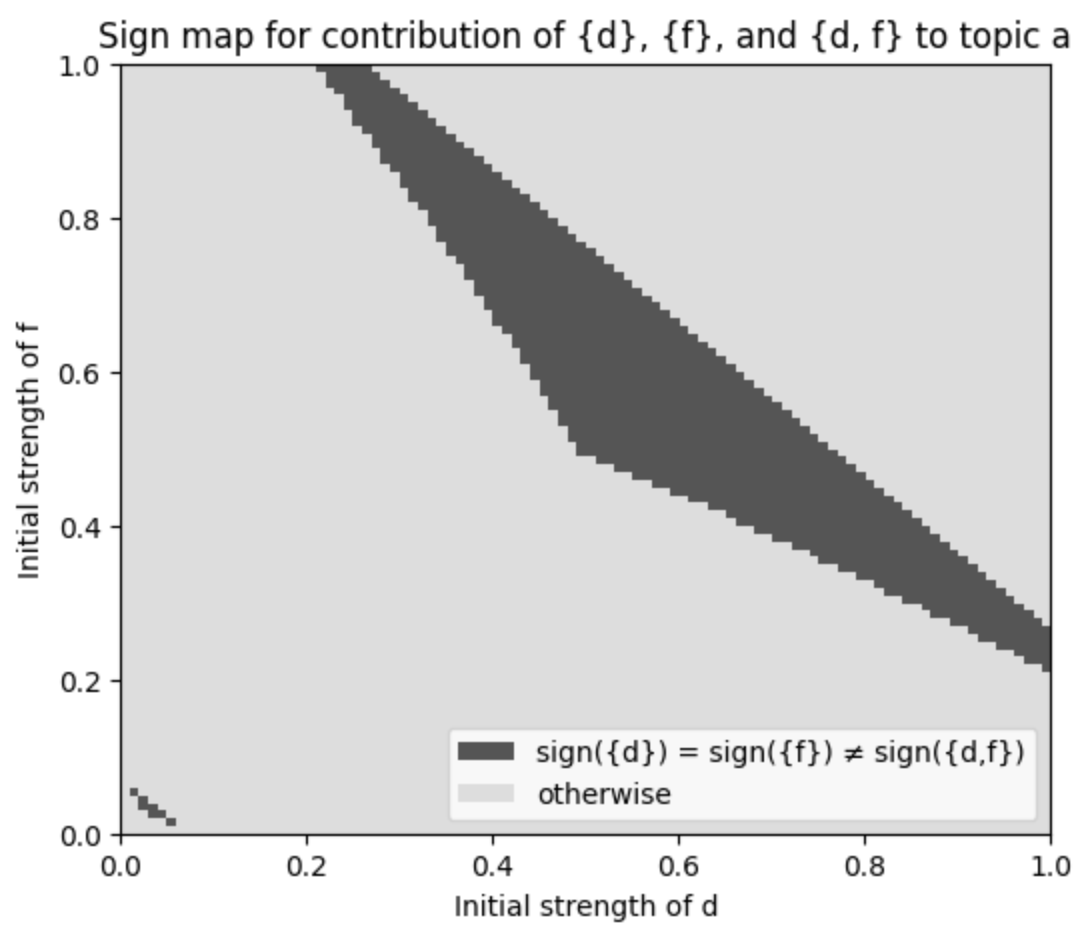}
        \label{fig:sign-map}
    }
    \caption{Example: the sign of the contributions of individual arguments to a topic may be inconsistent with the sign of the set contribution given the set of exactly these arguments.}
\end{figure}
Having introduced set contribution functions, we define principles for them, some of which are generalizations of single-argument contribution function principles, and some of which are novel (Section~\ref{sec:principles}). The principles based on the single-argument setting adapt established principles to sets:
\begin{enumerate*}[label=\roman*)]
    \item \emph{Contribution existence} stipulates that if a topic's final strength does not equal its initial strength, a non-zero set contributor must exist;
    \item \emph{Quantitative contribution existence} requires that for all partitions of a set of arguments, the sum of the set contributions to a topic, given the partitions, must add up to the difference between the topic's final and initial strengths;
    \item \emph{Directionality} stipulates that if a set contributor cannot reach a topic, the contribution to the topic must be zero;
    \item finally, \emph{counterfactuality} requires that a set contribution matches the inverted sign of the effect on a topics' final strength when removing the set contributor, or in the stricter \emph{quantitative} case, equals the difference between a topic's final strength with set contributor present and its final strength without.
\end{enumerate*}

We also introduce entirely new principles based on intuitions about desired behaviors, given several set contributors and interactions between them:
\begin{enumerate*}[label=\roman*)]
    \item \emph{Weak quantitative contribution existence} stipulates that a partition of the arguments must exist for which the sum of the contributions of the partition to a topic adds up to the difference between the topics final and initial strengths;
    \item \emph{Consistency} requires that if two set's contributions to the same topic have the same sign, then the contribution of the sets' union (to the same topic) must have the corresponding sign as well;
    \item \emph{Monotonicity} stipulates that when one set contributor is a superset of (or equal to) another one, the contribution of the former must be at least as large as the latter (to the same topic);
    \item Finally, \emph{symmetry} requires that if a topic's direct supporters and attackers are equally strong (with respect to a one-to-one mapping of final strengths), then the contribution of their union to the topic is $0$.
\end{enumerate*}

Based on both groups of principles, we execute a principle-based analysis of the set contribution functions with respect to several gradual semantics (Section~\ref{sec:analysis}), whose results are summarized in Tables~\ref{table:principle-overview-1} and Tables~\ref{table:principle-overview-2}.
Finally, we sketch an application scenario of set contribution functions, based on a QBAG use-case from the literature (Section~\ref{sec:application}), and we conclude the paper with a brief discussion of related and potential future work (Section~\ref{sec:discussion}).
An implementation of our contribution functions, alongside documentation and examples, is available at \url{https://github.com/TimKam/Quantitative-Bipolar-Argumentation}.

\begin{table}
\centering
\caption{Satisfaction and violation of set contribution function principles (cf. Section~\ref{sec:analysis}), given a semantics $\fs$ (cf. Section \ref{sec:prelim}) and a set contribution function $\sctrbempty$ (cf. Section \ref{sec:sctrb-functions}). All set contribution function principles in this table are generalizations of contribution function principles introduced in~\cite{DBLP:journals/ijar/KampikPYCT24}.}
\label{table:principle-overview-1}
\begin{tabular}{ |c|c|c|c|c| c| }
\hline
\textbf{$\sctrbempty$ / $\sigma$} & \textbf{QE}  & \textbf{DFQuAD} & \textbf{SD-DFQuAD} & \textbf{EB}  & \textbf{EBT} \\
\hline
\multicolumn{6}{|c|}{\textbf{Contribution Existence}} \\
\hline
$\sctrbrempty$ & \cmark & \cmark & \cmark & \cmark & \cmark \\ \hline
$\sctrbriempty$ & \cmark & \cmark & \cmark & \cmark & \cmark \\ \hline
$\sctrbsempty$ & \cmark & \cmark & \cmark & \cmark & \cmark \\ \hline
$\sctrbgmempty$ & \cmark & \xmark & \xmark & \cmark & \xmark \\ \hline
\hline
\multicolumn{6}{|c|}{\textbf{Quantitative Contribution Existence}} \\
\hline
$\sctrbrempty$ & \xmark & \xmark & \xmark & \xmark & \xmark \\ \hline
$\sctrbriempty$ & \xmark & \xmark & \xmark & \xmark & \xmark \\ \hline
$\sctrbsempty$ & \xmark & \xmark & \xmark & \xmark & \xmark \\ \hline
$\sctrbgmempty$ & \xmark & \xmark & \xmark & \xmark & \xmark \\ \hline
\hline
\multicolumn{6}{|c|}{\textbf{Directionality}} \\
\hline
$\sctrbrempty$ & \cmark & \cmark & \cmark & \cmark & \cmark \\ \hline
$\sctrbriempty$ & \cmark & \cmark & \cmark & \cmark & \cmark \\ \hline
$\sctrbsempty$ & \cmark & \cmark & \cmark & \cmark & \cmark \\ \hline
$\sctrbgmempty$ & \cmark & \cmark & \cmark & \cmark & \cmark \\ \hline
\hline
\multicolumn{6}{|c|}{\textbf{(Quantitative) Counterfactuality}} \\
\hline
$\sctrbrempty$ & \cmark & \cmark & \cmark & \cmark & \cmark \\ \hline    
$\sctrbriempty$ & \xmark & \xmark & \xmark & \xmark & \xmark \\ \hline
$\sctrbsempty$ & \xmark & \xmark & \xmark & \xmark & \xmark \\ \hline
$\sctrbgmempty$ & \xmark & \xmark & \xmark & \xmark & \xmark \\ \hline
\end{tabular}
\end{table}

\begin{table}
\centering
\caption{Satisfaction and violation of set contribution function principles, given a semantics $\fs$ and a set contribution function $\sctrbempty$. The set contribution function principles in this table are not generalizations of existing contribution function principles.}
\label{table:principle-overview-2}
\begin{tabular}{ |c|c|c|c|c| c| }
\hline
\textbf{$\sctrbempty$ / $\sigma$} & \textbf{QE}  & \textbf{DFQuAD} & \textbf{SD-DFQuAD} & \textbf{EB}  & \textbf{EBT} \\
\hline
\multicolumn{6}{|c|}{\textbf{Weak Quantitative Contribution Existence}} \\
\hline
$\sctrbrempty$ & \cmark & \cmark & \cmark & \cmark & \cmark \\ \hline
$\sctrbriempty$ & \cmark & \cmark & \cmark & \cmark & \cmark \\ \hline
$\sctrbsempty$ & \cmark & \cmark & \cmark & \cmark & \cmark \\ \hline
$\sctrbgmempty$ & \xmark & \xmark & \xmark & \xmark & \xmark \\ \hline
\hline
\multicolumn{6}{|c|}{\textbf{Consistency}} \\
\hline
$\sctrbrempty$ & \xmark & \xmark & \xmark & \xmark & \xmark \\ \hline
$\sctrbriempty$ &\xmark & \xmark & \xmark & \xmark & \xmark \\ \hline
$\sctrbsempty$ & \xmark & \xmark & \xmark & \xmark & \xmark \\ \hline
$\sctrbgmempty$ & \cmark & \cmark & \cmark & \cmark & \cmark \\ \hline
\hline
\multicolumn{6}{|c|}{\textbf{Monotonicity}} \\
\hline
$\sctrbrempty$ & \xmark & \xmark & \xmark & \xmark & \xmark \\ \hline
$\sctrbriempty$ & \xmark & \xmark & \xmark & \xmark & \xmark \\ \hline
$\sctrbsempty$ & \xmark & \xmark & \xmark & \xmark & \xmark \\ \hline
$\sctrbgmempty$ & \cmark & \cmark & \cmark & \cmark & \cmark \\ \hline
\hline
\multicolumn{6}{|c|}{\textbf{Symmetry}} \\
\hline
$\sctrbrempty$ & \cmark & \cmark & \cmark & \cmark & \cmark \\ \hline
$\sctrbriempty$ & \xmark & \xmark & \xmark & \xmark & \xmark \\ \hline
$\sctrbsempty$ & \xmark & \xmark & \xmark & \xmark & \xmark \\ \hline
$\sctrbgmempty$ & \cmark & \xmark & \xmark & \xmark & \xmark \\ \hline
\hline
\end{tabular}
\end{table}

\section{Preliminaries}
\label{sec:prelim}
Before we present the paper's contributions, let us introduce the necessary preliminaries: first of QBAG fundamentals and then of QBAG (single-argument) contribution functions and their principles.

\subsection{QBAGs and Gradual Semantics}
\label{subsec:qbags}
We use $\interval$ to denote \emph{strength values}.
Here, we assume a real interval, and more specifically $\mathbb{I} = [0, 1]$.
A QBAG is a set of elements, called \emph{arguments}, on which two disjoint binary relations, called \emph{support} and \emph{attack}, are established.
An initial strength strength function maps arguments to values in $\interval$.
\begin{definition}[Quantitative Bipolar Argumentation Graph (QBAG)~\cite{Potyka:2019,Baroni:Rago:Toni:2019}]
A \emph{Quantitative Bipolar Argumentation Graph (QBAG)} is a quadruple $\QBAG$, where $\Args$ is a set of arguments, 
$\Att \subseteq \Args \times \Args$ is called an \emph{attack} relation, 
$\Supp \subseteq \Args \times \Args$ is called a \emph{support} relation, $\Att \cap \Supp = \emptyset$ holds,
and $\is : \Args \to \interval$ is an \emph{initial strength function}.
\end{definition}
Henceforth, we assume a QBAG $\graph = \QBAG$ s.t. $\Args$ is finite, $\is$ is total, and $(\Args, \Att \cup \Supp)$ is acyclic (i.e., there is no set $\{\arga_0, \dots, \arga_n\} \subseteq \Args$ with $n \geq 1$, $\arga_0 = \arga_n$, and $(\arga_i, \arga_{i+1}) \in \Att \cup \Supp$ for all $i \in \{0, \dots, n-1\}$)).
We also assume that $\arga, \argb, \argx, \argy \in \Args$ and $S \subseteq \Args$.
If $(\arga, \argb) \in \Att$ we say that ``$\arga$ attacks $\argb$ (in $\graph$)'' or ``$\arga$ is an attacker of $\argb$ (in $\graph$)''.
If $(\arga, \argb) \in \Supp$ we say that ``$\arga$ supports $\argb$ (in $\graph$)'' or ``$\arga$ is a supporter of $\argb$ (in $\graph$)''.
We say that ``$\arga$ can reach $\argb$ (in $\graph$)'' iff there is a directed path from $\arga$ to $\argb$ in $(\Args, \Att \cup \Supp)$ (i.e., for a set $\{\arga_0, \dots, \arga_n \} \subseteq \Args$ s.t. $\arga_0 = \arga$ and $\arga_n = \argb$ it holds that $(\arga_i, \arga_{i+1}) \in \Att \cup \Supp$ for all $i \in \{0, \dots, n-1\}$), or $\arga = \argb$.
We may drop the reference to $\graph$ where the context is clear.

We define the \emph{restriction} of $\graph$ to $S$, denoted by $\graph\downarrow_{S}$, as follows: $\graph\downarrow_{S} := \left(S, \tau \cap (S \times \interval), \Att \cap (S \times S), \Supp \cap (S \times S) \right)$.
Somewhat analogously, we define the \emph{initial strength modification} of $\graph$ w.r.t. $\argx \in \Args$ and $\epsilon \in \interval$, denoted by $\graph\downarrow_{\tau(\argx) \leftarrow \epsilon}$ as follows: $\graph\downarrow_{\tau(\argx) \leftarrow \epsilon} := (\Args, \tau', \Att, \Supp)$,
where $\tau'(\argx) = \epsilon$ and $\tau'(\argy) = \tau(\argy)$ for all other $\argy \in \Args \setminus \{\argx\}$.

\emph{Gradual semantics} update the initial strengths of arguments in a QBAG to \emph{final} strengths, considering the QBAG's topology, thus giving rise to \emph{final strength functions}.
\begin{definition}[Gradual Semantics and Strength Functions~\cite{Baroni:Rago:Toni:2019,Potyka:2019}]\label{semantics-strength-functions}
A \emph{gradual semantics} $\fs$ defines a \emph{strength function} $\fs_{\graph} : \Args \to \interval \cup \{ \bot \}$ that assigns the \emph{final strength} $\fs_{\graph}(\argx)$ to each $\argx \in \Args$, 
where $\bot$ is a reserved symbol for ``undefined''. 
\end{definition}
Again, we may drop $\graph$ where the context is clear. 

A commonly studied class of gradual semantics are so-called \emph{modular semantics}~\cite{mossakowski2018modular}.
In order to draw inferences from QBAGs, modular semantics update arguments' initial strengths to final strengths by traversing the graph in topological order and updating a given argument's final strength by first \emph{aggregating} the final strengths of its attackers and supporters to then determine the \emph{influence} of the aggregation result on the argument's initial strength.
Below, we describe how modular semantics work based on~\cite{DBLP:journals/ijar/KampikPYCT24}.
Initially, we let $\langle \argx_1, ..., \argx_{|\Args|} \rangle$ be a fixed total order of the arguments in $\Args$ such that for every edge $(\argx_i, \argx_j) \in \Att \cup \Supp$ we have $i < j$ (i.e., a topological ordering of the acyclic QBAG).
We then define a vector $s^{(0)} \in \mathbb{R}^{|\Args|}$ of initial strengths of these arguments, i.e., $s^{(0)}_i := \tau(\argx_i)$ for $1 \leq i \leq |\Args|$.
The initial strengths are then iteratively updated to final strengths in two steps, at a given iteration step $i$:
\begin{description}
    \item[Aggregation.] An aggregation function $\alpha$ aggregates the current strength values of $\argx_i$'s supporters and attackers given $s^{(i)}$ to a single value. For representing attack and support, or lack thereof, to a given argument, we utilize a vector $v \in \{-1, 0, 1\}^n$, where $-1$, $1$, and $0$ represent attack, support, and neither, respectively.
    \item[Influence.] An influence function $\iota$ determines a new strength value for $\argx_i$, based on the aggregation result and $\argx_i$'s initial strength $w := \is(\argx_i)$.
\end{description}
The execution of aggregation and influence function for all arguments yields a new vector $s^{(i+1)}$; this update is executed until we reach $s^{(|\Args|)}$ (where $|\Args|$ is an upper bound on the number of updates, although depending on the topology of the QBAG, fewer updates may suffice).
Table~\ref{table:semantics} lists aggregation and influence functions from the literature, and Table~\ref{table:semanticsExamples} lists some modular semantics that make use of these functions.
\begin{table}[ht]
\centering
\footnotesize{
\begin{tabular}{lll}
\hline
\multicolumn{3}{c}{\textbf{Aggregation Functions}} \\ \hline
Sum & $\alpha^{\Sigma}_{v}: [0,1]^n \rightarrow \mathbb{R}$ & $\alpha^{\Sigma}_{v}(s) = \sum_{i = 1}^n v_i \times s_i $  \\
Product & $\alpha^{\Pi}_{v}: [0,1]^n \rightarrow [-1, 1]$ & $\alpha^{\Pi}_{v}(s) = \prod_{i:v_i=-1} (1 - s_i) - \prod_{i:v_i=1} (1 - s_i)$  \\
Top & $\alpha^{max}_{v}: [0,1]^n \rightarrow [-1, 1]$ & $a_v^{max}(s) = M_v(s) - M_{-v}(s),$ \\
 &  & where $M_v(s) = max\{0,v_1 \times s_1, \dots, v_n \times s_n\}$  \\
\hline
\multicolumn{3}{c}{\textbf{Influence Functions}} \\ \hline
Linear($k$) & $\iota^{l}_{w}: [-k, k] \rightarrow [0, 1]$ & $\iota^{l}_{w}(s) = w - \frac{w}{k} \times max\{0,-s\} + \frac{1-w}{k} \times max\{0, s\}$ \\
Euler-based & $\iota^{e}_{w}: \mathbb{R} \rightarrow [w^2, 1]$ & $\iota^{e}_{w}(s) = 1 - \frac{1-w^2}{1 + w \times e^s}$ \\
p-Max($k$) & $\iota^{p}_{w}: \mathbb{R} \rightarrow [0, 1]$ & $\iota^{p}_{w}(s) = w - w \times h(- \frac{s}{k}) + (1-w) \times h(\frac{s}{k}),$  \\
for $p \in \mathbb{N}$ &  & where $h(x) = \frac{max\{0,x\}^p}{1 + max\{0,x\}^p}$  \\
\end{tabular}}
\caption{Some aggregation $\alpha$ and influence $\iota$ functions~\cite[pp.~1724 Table 1]{Potyka:2019} for modular semantics. $s \in [0,1]^n$ is a strength vector 
(associating each argument with its current strength);
$v \in \{-1,0,1\}^n$ is a relationship vector indicating
which arguments attack ($-1$), support ($1$) or are in no 
relationship to ($0$) the argument of interest; $w$ is an initial strength.
}
\label{table:semantics}
\end{table}

\begin{table}[ht]
\centering
\begin{tabular}{lll}
\hline
\textbf{Semantics}           & \textbf{Aggregation} & \textbf{Influence}  \\ \hline
QuadraticEnergyModel (QE)        & Sum         & 2-Max(1)  \\
DFQuADModel (DFQuAD)       & Product     & Linear(1) \\
SquaredDFQuADModel (SD-DFQuAD) & Product     & 1-Max(1)  \\
EulerBasedModel (EB)   & Sum         & EulerBased  \\
EulerBasedTopModel (EBT) & Top         & EulerBased 
\end{tabular}
\caption{Examples of modular (gradual) semantics.}
\label{table:semanticsExamples}
\end{table}
To support the analysis of gradual semantics, \emph{principles} have been defined that describe desired semantics behavior.
While most of our observations do not rely on previously established principles, we use the \emph{bi-variate independence} and \emph{bi-variate directionality} principles~\cite{AMGOUD201839} for some results with respect to \emph{contribution existence}.
Bi-variate independence stipulates that an argument's final strength should not be affected by arguments that are not connected to it in the graph (i.e., there is no path between them even when attack/support directions are ignored).
\begin{sprinciple}[Bi-variate Independence]
\label{sprinciple:bi-variate-independence}
    A gradual semantics $\fs$ satisfies the \emph{bi-variate independence} principle iff for all QBAGs $\graph = \QBAG$, $\graph' = \QBAGp$ with $\Args \cap \Args' = \emptyset$, and 
    $\graph^* = (\Args \cup \Args', \is \cup \is', \Att \cup \Att', \Supp \cup \Supp')$ 
    it holds that $\fs_{\graph}(\arga) = \fs_{\graph^*}(\arga)$ for all $\arga \in \Args$.
\end{sprinciple}
The bi-variate directionality principle requires that an argument's final strength depend only on \emph{incoming} attacks and supports, not on the arguments it attacks and supports.
\begin{sprinciple}[Bi-variate Directionality]
\label{sprinciple:bi-variate-directionality}
    A gradual semantics $\fs$ satisfies the \emph{bi-variate directionality} principle iff for every $\graph = \QBAG$, every $(\argb,\argc) \notin (\Att \cup \Supp)$, and every $\graph'$ obtained from $\graph$ by adding $(\argb,\argc)$ to either $\Att$ or $\Supp$, it holds that for all $\arga \in \Args$ if there is no directed path from $\argc$ to $\arga$ in $\graph'$, then $\fs_{\graph}(\arga) = \fs_{\graph'}(\arga)$.
\end{sprinciple}
All semantics in Table~\ref{table:semanticsExamples} satisfy both principles.
\begin{lemma}
\label{lemma:semantics-bi-variate-principles}
    QE, DFQuAD, SD-DFQuAD, EB, and EBT semantics $\fs$ satisfy the bi-variate independence and bi-variate directionality principles.
\end{lemma}
\begin{proof}
    We observe that:
    \begin{description}
        \item[Bi-variate independence] 
        is satisfied because, by definition of each semantics, an argument's final strength depends only on its own initial strength together with an aggregation over its attackers and supporters.
        Accordingly, adding a disjoint component has no effect on an argument's final strength.
        \item[Bi-variate directionality]
        is satisfied (for acyclic QBAGs) because an argument's final strength under these semantics depends only on its own initial strength and the final strengths of its attackers and supporters.
        Adding an edge $(\argb, \argc)$ can only affect $\argc$ and arguments that are (directly or indirectly) influenced by $\argc$;
        the final strength of any argument $\arga$ such that there is no directed path from $\argc$ to $\arga$ does not change.
    \end{description}
\end{proof}
\subsection{Single-Argument Contribution Functions and their Principles}
\label{subsec:ctrb-functions-principles}
In this paper, we generalize (single-argument) contribution functions that determine the contribution of one argument---the so-called \emph{contributor}---to another one (the \emph{topic}). Henceforth, we assume a gradual semantics $\fs$.
\begin{definition}[(Single-Argument) Contribution Function]
\label{def:ctrb-function}
A contribution function
$$\ctrbempty_{\fs, \graph, \arga}: \Args \rightarrow \mathbb{R} \cup {\bot}$$
with respect to $\graph = \QBAG$, $\fs$, and $\arga \in \Args$ takes an argument $\argx \in \Args$ and returns a real number or \emph{undefined} ($\bot$), quantifying the \emph{contribution of (contributor) $\argx$ to (topic) $\arga$.}
\end{definition}
We may drop any of the subscripts $\fs$, $\graph$, and $\arga$ where the context is clear.

A simple and intuitive contribution function computes the difference between the topic's final strength given the original QBAG and its final strength in absence of the contributor.
\begin{definition}[Removal-based Contribution Function~\cite{DBLP:journals/ijar/KampikPYCT24}]
\label{def:removal}
We define the removal-based contribution function $\mathsf{Ctrb}^{\cal R}_{\arga}: \Args \rightarrow \mathbb{R} \cup {\bot}$ as follows:
\begin{align*}
   \ctrbr{\argx}{\arga} := \fs_\graph(\arga) - \fs_{\graph\downarrow_{\Args \setminus \{\argx\}}}(\arga).
\end{align*}
\end{definition}
One may argue that the removal-based contribution function does not determine the \emph{intrinsic} contribution of the contributor to the topic's final strength, as it also measures the impact of the contributor's contributors.
In order to avoid this, one can remove incoming attack and support edges to the contributor, thus ignoring effects other arguments have on the contributor's final strength (given reasonable assumptions about the applied semantics).
\begin{definition}[Intrinsic Removal-based Contribution Function~\cite{DBLP:journals/ijar/KampikPYCT24}]
\label{def:intrinsic-removal}
We define the intrinsic removal-based contribution function $\mathsf{Ctrb}^{{\cal R}'}_{\arga}: \Args \rightarrow \mathbb{R} \cup {\bot}$ as follows:
\begin{align*}
\ctrbri{\argx}{\arga} := \fs_{(\Args, \tau, \Att \setminus \{(\argy, \argx) | (\argy, \argx) \in \Att \}, \Supp \setminus \{(\argy, \argx) | (\argy, \argx) \in \Supp \})}(\arga) - \fs_{\graph\downarrow_{\Args \setminus \{\argx\}}}(\arga).
\end{align*}
\end{definition}
Yet differently, we can apply Shapley values for a principle-based game theoretical assessment of contributions to the final strength of a topic, where contributors are seen as agents forming coalitions in cooperative games.
\begin{definition}[Shapley Value-based Contribution Function~\cite{DBLP:journals/ijar/KampikPYCT24}]
\label{def:shapley}
We define the Shapley value-based contribution function $\mathsf{Ctrb}^{\cal S}_{\arga}: \Args \rightarrow \mathbb{R} \cup {\bot}$ as follows:
\begin{align*}
&{} \ctrbs{\argx}{\arga} :=  \nonumber \\
&{} \quad \sum_{S \subseteq \Args \setminus \{ \argx, \arga \}} \frac{|S|! \cdot (|\Args \setminus \{\arga\}|-|S|-1)!}{|\Args \setminus \{\arga\}|!}\left(\operatorname{\fs}_{\graph\downarrow_{\Args \setminus S}}(\arga) - \operatorname{\fs}_{\graph\downarrow_{\Args \setminus (S \cup \{ \argx \})}}(\arga)\right).
\end{align*}
\end{definition}
Finally, we can measure the contribution of the effect of a marginal change to the initial strength of a contributor on the final strengths of a topic.
\begin{definition}[Gradient-based Contribution Function~\cite{DBLP:journals/ijar/KampikPYCT24}]
\label{def:gradient}
We define the gradient-based contribution function $\mathsf{Ctrb}^{\partial}_{\arga}: \Args \rightarrow \mathbb{R} \cup {\bot}$ as follows:
\begin{align*}
    \ctrbg{\argx}{\arga} := \frac{\partial f_{\arga}}{\partial \is(\argx)} \left(\is(\argx_1), \ldots, \is(\argx_n)\right),
\end{align*}
where ${\argx_1, ..., \argx_n}$ are all arguments from which there is a path to $\arga$ in the graph $(\Args, \Att \cup \Supp)$ and $f$ is a composition of aggregation function $\alpha$ and influence function $\iota$.
\end{definition}
For contribution functions, \emph{principles} have been defined that formalize desiderata of function behavior and can be used to select contribution functions for a given use-case.
Principle satisfaction typically depends on the behavior of the applied gradual semantics.
Below, we list and briefly explain some principles that have been defined in previous work.
Henceforth, we assume a generic contribution function $\ctrbempty$.

The \emph{contribution existence} principle stipulates that each (topic) argument whose final strength does not equal its initial strength should have a contributor with a non-zero contribution.
\begin{principle}[Contribution Existence~\cite{DBLP:journals/ijar/KampikPYCT24}]
\label{principle:existence}
$\ctrbempty$ satisfies the \emph{contribution existence principle} w.r.t.\ a gradual semantics $\fs$
iff  $\sigma(\arga) \neq \is(\arga)$ implies that $\exists \argx \in \Args \setminus \{\arga\}$ with $\ctrb{\argx}{\arga} \neq 0$.
\end{principle}
Quantitative contribution existence is a stricter principle, requiring that the sum of all contributors' contribution must add up to the difference between final and initial strengths of the topic.
\begin{principle}[Quantitative Contribution Existence~\cite{DBLP:journals/ijar/KampikPYCT24}]
\label{principle:qexistence}
$\ctrbempty$ satisfies the \emph{quantitative contribution existence principle} w.r.t.\ a gradual semantics $\fs$
iff it holds that $\sum_{\argx \in \Args \setminus \{\arga\}} \ctrb{\argx}{\arga} = \sigma(\arga) - \tau(\arga)$.
\end{principle}
Directionality stipulates that contributors that cannot reach a topic must have a contribution of zero, following the intuition that contributions cannot propagate against a QBAG's topological order.
\begin{principle}[Directionality~\cite{Cyras:Kampik:Weng:2022}\footnote{We adjusted this principle for the case that $\argx = \arga$, as above we assume that an argument can reach itself.}]
\label{principle:directionality}
$\ctrbempty$ satisfies the \emph{directionality principle} w.r.t.\ a gradual semantics $\fs$
iff whenever there is no directed path in $\graph$ from $\argx \in \Args \setminus \{\arga\}$ to $\arga$, then $\ctrb{\argx}{\arga} = 0$.
\end{principle}

Counterfactuality requires that the direction of the contribution is consistent with direction of the difference between a topic's final strength in the original QBAG and its final strength in absence of the contributor.
\begin{principle}[Counterfactuality~\cite{DBLP:journals/ijar/KampikPYCT24}]
\label{principle:counterfactual}
$\ctrbempty$ satisfies the \emph{counterfactuality principle} w.r.t.\ a gradual semantics $\fs$
iff for every $\argx \in \Args \setminus \{\arga\}$\footnote{As a minor technical fix of the definition in~\cite{DBLP:journals/ijar/KampikPYCT24}, we explicitly require $\argx \neq \arga$ (same for the \emph{quantitative} case).} the following statements hold:
\begin{itemize}
    \item If $\ctrb{\argx}{\arga} < 0$, then
    $\fs_\graph(\arga) < \fs_{\graph \downarrow_{\Args \setminus \{\argx\}}}(\arga)$;
    \item If $\ctrb{\argx}{\arga} = 0$,
    then 
    $\fs_\graph(\arga) = \fs_{\graph \downarrow_{\Args \setminus \{\argx\}}}(\arga)$;
    \item If $\ctrb{\argx}{\arga} > 0$,
    then 
    $\fs_\graph(\arga) > \fs_{\graph \downarrow_{\Args \setminus \{\argx\}}}(\arga)$.
\end{itemize}
\end{principle}
Finally, quantitative counterfactuality is a stricter principle, stipulating that a contribution must equal the difference between a topic's final strength in the original QBAG and its final strength in absence of the contributor.
\begin{principle}[Quantitative Counterfactuality~\cite{DBLP:journals/ijar/KampikPYCT24}]
\label{principle:q-counterfactual}
    $\ctrbempty$ satisfies the \emph{quantitative counterfactuality principle} w.r.t.\ a gradual semantics $\fs$ iff for every $\argx \in \Args \setminus \{\arga\}$ it holds that $\ctrb{\argx}{\arga} = \fs_\graph(\arga) - \fs_{\graph \downarrow_{\Args \setminus \{\argx\}}}(\arga)$.
\end{principle}

\section{Set Contribution Functions}
\label{sec:sctrb-functions}
A set contribution function takes a subset of a QBAG's arguments in order to quantify these arguments' contribution to a topic.
\begin{definition}[Set Contribution Function]
\label{def:set-ctrb-function}
A set contribution function
$$\sctrbempty_{\fs, \graph, \arga}: 2^\Args \rightarrow \mathbb{R} \cup {\bot}$$
with respect to $\graph = \QBAG$, $\fs$, and $\arga \in \Args$ takes a set of arguments $\argX \subseteq \Args$ and returns a real number or \emph{undefined} ($\bot$), quantifying the \emph{contribution of (set contributor) $\argX$ to (topic) $\arga$.}
\end{definition}
Given $\sctrb{\argX}{\arga}$, we call $\arga$ \emph{topic} and $\argX$ \emph{set contributor}.
Given a contribution function $\sctrbempty_{\fs, \graph, \arga}$, we drop the subscripts $\fs, \graph, \arga$ where the context is clear.
Below, we assume a QBAG $\graph$, a topic argument $\arga \in \Args$, a set contributor $\argX \subseteq \Args \setminus \{\arga\}$, and a strength function $\fs$, if not specified otherwise.

The following definitions are based on those presented in \cite{DBLP:journals/ijar/KampikPYCT24}, and have been generalized to quantify the contribution of a set of arguments, rather than a single contributing argument.

The removal-based set contribution function determines the difference between the topic's final strength in the original graph and its final strength in the restriction of the graph from which the set contributor has been removed.
\begin{definition}[Removal-Based Set Contribution Function]
\label{def:set-removal}
    The removal-based set contribution function $\sctrbrempty: 2^\Args \rightarrow \mathbb{R} \cup {\bot}$ is defined as:
    \begin{align}\label{eq:set-removal}
        \sctrbr{\argX}{\arga} := \fs_\graph(\arga) - \fs_{\graph\downarrow_{\Args \setminus \argX}}(\arga) 
    \end{align}
\end{definition}
Intuitively, the \emph{intrinsic} variant of the removal-based set contribution function removes all incoming attacks and supports to the set contributor before determining the final strength of the topic, and is otherwise analogous to the removal-based set contribution function.
The idea is to control for indirect contributions that come through attackers or supporters of the set contributor, so the resulting contribution is intrinsic with respect to $\Args \setminus \argX$, since $\argX$ is isolated from effects of the rest of the graph $\graph\downarrow_{\Args \setminus \argX}$ while internal attacks and supports within $\argX$ are preserved.
\begin{definition}[Intrinsic Removal-Based Set Contribution Function]
\label{def:set-intrinsic-removal}
    The intrinsic removal-based set contribution function $\sctrbriempty: 2^\Args \rightarrow \mathbb{R} \cup {\bot}$ is defined as:
    \begin{align*}\label{eq:set-intrinsic-removal}
        \sctrbri{\argX}{\arga} & := \nonumber \\
        & \hspace{-0.5cm} \fs_{(\Args, \is, \Att \setminus \{(\argy, \argx) \in \Att | \argy \notin \argX, \argx  \in \argX \}, \Supp \setminus \{(\argy, \argx) \in \Supp | \argy \notin \argX, \argx  \in \argX \})}(\arga) - \fs_{\graph\downarrow_{\Args \setminus \argX }}(\arga)
    \end{align*}
\end{definition}
Gradient set-based contribution functions determine and then combine the effects of marginal changes to the initial strengths of arguments in the set contributor.
Accordingly, we first define an \emph{abstract} gradient-based set contribution function that we can later instantiate given a specific pooling function. 
\begin{definition}[Gradient-Based Set Contribution Function (Abstract)]
\label{def:set-gradient-abstract}
    Let $\arga \in \Args$ and let $(y_1, \ldots, y_n) \in \Args$ be the arguments that have a directed path to $\arga$ in $\graph$, with corresponding initial strengths $(\is(y_1), \ldots, \is(y_n))$. Now, let $\psi \colon \mathbb{R}^{|\argX|} \to \mathbb{R}$ be a pooling function that maps a set of $|\argX|$ real numbers to a single real value. The gradient-based set contribution function $\sctrbgempty: 2^\Args \rightarrow \mathbb{R} \cup {\bot}$ is then defined as:
        \begin{align*}
            \sctrbg{\argX}{\arga} := \psi \left( \left\{ \frac{\partial f_{\arga}}{\partial \is(\argx)} (\is(\argy_1), \ldots, \is(\argy_n)) \Bigm| x \in X \right\} \right).
        \end{align*}
\end{definition}
Intuitively, a suitable pooling function is $\max$, as it helps us answer the question how to affect the maximal (positive) change, assuming a maximization objective.
\begin{definition}[Gradient-Based Set Contribution Function (Max)]
\label{def:set-gradient-max}
    Let $\arga \in \Args$ and let $(y_1, \ldots, y_n) \in \Args$ be the arguments that have a directed path to $\arga$ in $\graph$, with corresponding initial strengths $(\is(y_1), \ldots, \is(y_n))$. The max gradient-based set contribution function $\sctrbgmempty: 2^\Args \rightarrow \mathbb{R} \cup {\bot}$ is now defined as:
    \begin{align*}
        \sctrbgm{\argX}{\arga} := \max_{x \in X}\frac{\partial f_{\arga}}{\partial \is(\argx)} \left(\is(\argy_1), \ldots, \is(\argy_n)\right).
    \end{align*}
    This function captures the maximum gradient-based contribution of any argument in $\argX$ to the topic argument $\arga$.
\end{definition}
Note that similarly, one could use $\min$, in case our objective is to minimize the topic's final strength, or the $\max$ of the gradients' absolutes, in case we are interested in how ``stable'' (in some sense) the topic's final strength is, given marginal changes to the initial strength of an argument in the set contributor.
Still, for the sake of conciseness, we limit ourselves to $\max$.

Finally, we introduce the Shapley value-based set contribution function that utilizes the well-known power index for fairly distributing the contributions among arguments and sets thereof.
We can first provide a general definition that assumes a specific partition $P$ of $\Args \setminus \{\arga\}$ and then determines the contribution of $\argX \in P$.
\begin{definition}[Shapley Value-Based Set-Partition Contribution Function]
\label{def:partition-shapley}
Let $P$ be a partition of $\Args \setminus \{\arga\}$ and $\argX \in P$. the Shapley value-based set-partition contribution function is defined as:
    \begin{align*}
        &{} \pctrbs{\argX,P}{\arga} :=  \nonumber \\
        &{} \quad \sum_{P' \subseteq P \setminus \{\argX\}} \frac{|P'|! \cdot (|P|-|P'|-1)!}{|P|!}\left(\operatorname{\fs}_{\graph\downarrow_{\Args \setminus P^*}}(\arga) - \operatorname{\fs}_{\graph\downarrow_{\Args \setminus (P^* \cup \argX)}}(\arga)\right),
    \end{align*}
    where $P^* = \bigcup_{S \in P'} S$.
\end{definition}
For a fixed partition $P$, each block $\argX \in P$ can simply be treated (and relabeled) as a single player, so $\pctrbsempty$ is simply the standard Shapley value applied on the player set $P$, and all axioms from \cite{shapley1953value} apply directly.
Clearly, the function above is not a set contribution function, as it depends---in addition to the set contributor---on a specific partition.
Hence, we also provide a specialized set contribution function variant that partitions $\Args \setminus \{\arga\}$ into the set contributor $\argX$ and sets of single arguments.
This treats the designated set contributor $X$ as one player and every other argument as its own player, which is a natural default when no additional grouping information over the arguments is given.
However, the set-partition contribution function above can potentially be useful due to its increased expressivity.
\begin{definition}[Shapley Value-Based Set Contribution Function]
\label{def:set-shapley}
    The Shapley value-based set contribution function $\sctrbsempty: 2^\Args \rightarrow \mathbb{R} \cup {\bot}$ is defined as:
    \begin{align*}
    &{} \sctrbs{\argX}{\arga} :=  \nonumber \\
    &{} \quad \sum_{S \subseteq \Args \setminus  (\argX \cup \{\arga\})} \frac{|S|! \cdot (|\Args \setminus (\{\arga\} \cup \argX)|-|S|)!}{(|\Args \setminus (\{\arga\} \cup \argX)| + 1)!}\left(\operatorname{\fs}_{\graph\downarrow_{\Args \setminus S}}(\arga) - \operatorname{\fs}_{\graph\downarrow_{\Args \setminus (S \cup \argX)}}(\arga)\right).
    \end{align*}
\end{definition}
Compared with the classical ``single-argument'' Shapley value, grouping arguments into a set or partition reduces the number of players (in the game theoretical sense) from $|\Args \setminus \{a\}|$ to $|\Args \setminus (\{\arga\} \cup \argX)| + 1$, which in turn reduces the number of coalitions that must be considered when computing the Shapley value.
Therefore, when a meaningful grouping of arguments into a set contributor is available, the Shapley value-based set contribution function can also be used to reduce computational costs and may render otherwise infeasible computations feasible.

In terms of computational cost, we observe that:
$\sctrbrempty$ evaluates the final strength of the topic argument twice with one graph modification;
$\sctrbriempty$ evaluates the final strength of the topic argument twice with two graph modifications;
$\pctrbsempty$ evaluates the final strength of the topic argument twice with one graph modification per coalition, of which there are $2^{|P|-1}$;
$\sctrbsempty$ evaluates the final strength of the topic argument twice with one graph modification per coalition, of which there are $2^{(|\Args \setminus (\{\arga\} \cup \argX)| + 1) - 1}$;
and $\sctrbgmempty$ evaluates the final strength of the topic argument twice per argument in the set contributor (i.e., $2 \cdot |X|$ evaluations) with no graph modifications, since the gradient-based set contribution function takes a fundamentally different approach by modifying the initial strength of the attackers/supporters instead.
We assume here that $\sctrbgmempty$ is computed via a finite-difference approximation of the partial derivatives. Accordingly, for each argument in the set contributor, the final strength of the topic argument is evaluated at the current initial strength and once more after an $\varepsilon$-perturbation, resulting in up to two evaluations per argument.

\section{Generalization of Single-Argument Contribution Functions}
\label{sec:generalization}
The set contribution functions that we have introduced in the previous section are supposed to generalize the single-argument contribution functions as presented in~\cite{DBLP:journals/ijar/KampikPYCT24}.
Intuitively, a set contribution function generalizes a single-argument contribution function only if it generally holds that a contributor's single-argument contribution to a topic is the same as the corresponding set contribution to that topic, given a set that only contains the aforementioned contributor.
\begin{principle}[$\ctrbempty$-Generalization]
    \label{principle:ctrb-func-generalization}
    A set contribution function $\sctrbempty$ generalizes a contribution function $\ctrbempty$ iff $\ctrb{\argx}{\arga} = \sctrb{\{\argx\}}{\arga}$.
\end{principle}
We can show that our set contribution functions indeed generalize their single-argument counterparts.
\begin{proposition}
\label{prop:ctrb-function-generalization}
     $\sctrbrempty$ generalizes $\ctrbrempty$, $\sctrbriempty$ generalizes $\ctrbriempty$, $\sctrbsempty$ generalizes $\ctrbsempty$, and $\sctrbgmempty$ generalizes $\ctrbgempty$.
\end{proposition}
\begin{proof}
    For every set contribution function $\sctrbempty \in~\{\sctrbrempty, \sctrbriempty, \sctrbsempty, \sctrbgmempty\}$, we replace $\argX$ by $\{\argx\}$ in the corresponding definition (Definitions~\ref{def:set-removal},~\ref{def:set-intrinsic-removal},~\ref{def:set-shapley}, and~\ref{def:set-gradient-max}, respectively).
    \begin{description}
        \item[Removal-based.] This yields $\fs_\graph(\arga) - \fs_{\graph\downarrow_{\Args \setminus \{\argx\}}}(\arga) = \ctrbr{\argx}{\arga}$ (Definition~\ref{def:removal}).
        \item[Intrinsic removal-based.] This yields $ \fs_{(\Args, \tau, \Att \setminus \{(\argy, \argx) | (\argy, \argx) \in \Att \}, \Supp \setminus \{(\argy, \argx) | (\argy, \argx) \in \Supp \})}(\arga) - \fs_{\graph\downarrow_{\Args \setminus \{\argx\}}}(\arga) = \ctrbri{\argx}{\arga}$ (Definition~\ref{def:intrinsic-removal}). Note that we can drop the ``$\argy \not \in \{\argx\}$'' constraints of Definition~\ref{def:set-intrinsic-removal}, as we assume that $\graph$ is acyclic.
        \item[Shapley value-based.] This yields:
        \begin{align*}
            &{} \sum_{S \subseteq \Args \setminus  (\{\argx\} \cup \{\arga\})} \frac{|S|! \cdot (|\Args \setminus (\{\arga\} \cup \{\argx\})|-|S|)!}{(|\Args \setminus (\{\arga\} \cup \{\argx\})| + 1)!}\left(\operatorname{\fs}_{\graph\downarrow_{\Args \setminus S}}(\arga) - \operatorname{\fs}_{\graph\downarrow_{\Args \setminus (S \cup \{\argx\})}}(\arga)\right) \\
            &{} = \sum_{S \subseteq \Args \setminus \{ \argx, \arga \}} \frac{|S|! \cdot (|\Args \setminus \{\arga\}|-|S|-1)!}{|\Args \setminus \{\arga\}|!}\left(\operatorname{\fs}_{\graph\downarrow_{\Args \setminus S}}(\arga) - \operatorname{\fs}_{\graph\downarrow_{\Args \setminus (S \cup \{ \argx \})}}(\arga)\right) \\
            &{}  = \ctrbs{\argx}{\arga} \text{ (Definition~\ref{def:shapley})}.
        \end{align*}
        \item[Gradient-based.] This yields: $\max_{\argx \in \argX}\frac{\partial f_{\arga}}{\partial \is(\argx)} \left(\is(\argy_1), \ldots, \is(\argy_n)\right) = \\\max_{\argx \in \{\argx\}} \frac{\partial f_{\arga}}{\partial \is(\argx)} \left(\is(\argy_1), \ldots, \is(\argy_n)\right) = \ctrbg{\argx}{\arga}$ (Definition~\ref{def:gradient}).
    \end{description}
\end{proof}
%

\section{Principles}
\label{sec:principles}
For the set contribution functions as introduced in the previous section, we introduce contribution function principles that reflect expected behaviors of these functions, and that are either satisfied or violated given a set contribution function and a gradual semantics.
First, we introduce principles that reflect the single-argument contribution function principles from the preliminaries.
We then formally show how these new set contribution function principles relate to their single-argument counterparts.
Finally, we introduce entirely new principles that reflect desiderata of function behavior given set partitions and other interactions between contributor sets.

Contribution existence stipulates that if a topic's final strength does not equal its initial strength, there must exist a set contributor with a non-zero contribution.
\begin{principle}[Contribution Existence]
\label{sprinciple:ctrb-existence}
    A set contribution function $\sctrbempty$ satisfies \emph{contribution existence} w.r.t. a gradual semantics $\fs$ iff ${\fs(\arga) \neq \is(\arga)}$ implies that $\exists \argX \subseteq \Args \setminus \{\arga\}$ with $\sctrb{\argX}{\arga} \neq 0$.
\end{principle}
Quantitative contribution existence requires that for all partitions of $\Args \setminus \{\arga\}$, the sum of the set contributions of the elements in the partition must amount to the difference between the topic's final and initial strengths.
\begin{principle}[Quantitative Contribution Existence]
\label{principle:quantitative-ctrb-existence}
    A set contribution function $\sctrbempty$ satisfies \emph{quantitative contribution existence} w.r.t. a gradual semantics $\fs$ iff for all partitions $P$ of $\Args \setminus \{\arga\}$ it holds that $\sum_{\argX \in P} \sctrb{\argX}{\arga} = \fs(\arga) - \is(\arga)$.
\end{principle}
Directionality stipulates that a set contributor must have a contribution of zero if none of its arguments can reach the topic.
\begin{principle}[Directionality]
\label{sprinciple:directionality}
    A set contribution function $\sctrbempty$ satisfies \emph{directionality} w.r.t. a gradual semantics $\fs$ iff for every $\argX \subseteq \Args \setminus \{\arga\}$ it holds that $\sctrb{\argX}{\arga} = 0$ whenever $\forall \argx \in \argX$ it holds that $\argx$ cannot reach $\arga$ (in $\graph$). 
\end{principle}
\begin{principle}[Counterfactuality]
\label{principle:counterfactuality}
    A set contribution function $\sctrbempty$ satisfies \emph{counterfactuality} w.r.t. a gradual semantics $\fs$ iff for every $\argX \subseteq \Args \setminus \{\arga\}$ the following statements hold:
    \begin{itemize}
        \item If $\sctrb{\argX}{\arga} < 0$, then $\fs_{\graph}(a) < \fs_{\graph\downarrow_{\Args \setminus \argX}}(\arga)$.
        \item If $\sctrb{\argX}{\arga} = 0$, then $\fs_{\graph}(a) = \fs_{\graph\downarrow_{\Args \setminus \argX}}(\arga)$.
        \item If $\sctrb{\argX}{\arga} > 0$, then $\fs_{\graph}(a) > \fs_{\graph\downarrow_{\Args \setminus \argX}}(\arga)$.
    \end{itemize}
\end{principle}
\begin{principle}[Quantitative Counterfactuality]
\label{principle:quantitative-counterfactuality}
    A set contribution function $\sctrbempty$ satisfies \emph{quantitative counterfactuality} w.r.t. a gradual semantics $\fs$ iff for every $\argX \subseteq \Args \setminus \{\arga\}$ it holds that $\sctrb{\argX}{\arga} = \fs_{\graph}(\arga) - \fs_{\graph\downarrow_{\Args \setminus \argX}}(\arga)$.
\end{principle}
Intuitively, one would expect that set contribution principles as specified above can be formally related to single-argument contribution function principles.
We can achieve this by introducing notions of \emph{principle generalization}, stipulating that principle satisfaction by a single argument contribution function implies satisfaction of its generalizations or vice versa. As a prerequisite to formalizing this, we denote a generic single-argument contribution function principle by $p$ (i.e., $p: {\cal C} \times {\cal Z} \rightarrow \{true, false\}$, where $\cal C$ is the class of all single-argument contribution functions, and a generic set contribution function principle by $p'$ (i.e., $p': {\cal S} \times {\cal Z} \rightarrow \{true, false\}$, where $\cal S$ is the class of all set contribution functions; in both cases ${\cal Z}$ is the class of all gradual semantics.
\begin{definition}[$\uparrow$ and $\downarrow$ Principle Generalization]
    \label{definition:generalization}
    Let $p'$ be a set contribution function principle.
    Iff for every gradual semantics $\fs$, for every single-argument contribution function $\ctrbempty$, and for every set contribution function $\sctrbempty$ s.t. $\sctrbempty$ generalizes $\ctrbempty$ it holds that:
    \begin{itemize}
        \item if $\ctrbempty$ satisfies $p$ w.r.t. $\sigma$ then $\sctrbempty$ satisfies $p'$ w.r.t. $\sigma$, then we say that ``$p'$ $\uparrow$-generalizes $p$'';
        \item if $\sctrbempty$ satisfies $p'$ w.r.t. $\sigma$ then $\ctrbempty$ satisfies $p$ w.r.t. $\sigma$, then we say that ``$p'$ $\downarrow$-generalizes $p$''.
    \end{itemize}
\end{definition}
We can show that the set contribution function principles that we have introduced so far all either $\uparrow$-generalize or $\downarrow$-generalize single-argument contribution function principles.
Contribution existence is a $\uparrow$-generalization.
\begin{proposition}\label{prop:up-generalization}
    The set contribution function principle \emph{contribution existence} $\uparrow$-generalizes its equally named single-argument contribution principle-counterpart.
\end{proposition}
\begin{proof}
     Assume a contribution function $\sctrbempty$ that generalizes a contribution function $\ctrbempty$. Then, it holds that $\ctrb{\argx}{\arga} = \sctrb{\{\argx\}}{\arga}$ (Definition~\ref{principle:ctrb-func-generalization}).
    We can assume $\sigma(\arga) \neq \is(\arga)$ and need to show that if $\exists \argx \in \Args \setminus \{\arga\}$ with $\ctrb{\argx}{\arga} \neq 0$ (Principle~\ref{principle:existence}) then $\exists \argX \subseteq \Args \setminus \{\arga\}$ with $\sctrb{\argX}{\arga} \neq 0$ (Principle~\ref{sprinciple:ctrb-existence}). Because $\ctrb{\argx}{\arga} = \sctrb{\{\argx\}}{\arga}$ holds, we know this is the case for every $\argx \in \Args \setminus \{\arga\}$ s.t. $\ctrb{\argx}{\arga} \neq 0$ (which must exist, due to the satisfaction of Principle~\ref{principle:existence}).
\end{proof}
In contrast, quantitative contribution existence, directionality, counterfactuality, and quantitative counterfactuality are $\downarrow$-generalizations.
\begin{proposition}
\label{prop:down-generalization}
The following set contribution function principles $\downarrow$-generalize their equally named single-argument contribution principle-counterparts: \emph{quantitative contribution existence}, \emph{directionality}, \emph{counterfactuality}, and \emph{quantitative counterfactuality}.
\end{proposition}
\begin{proof}
    As in the proof of Proposition~\ref{prop:up-generalization}, assume a set contribution function $\sctrbempty$ that generalizes a contribution function $\ctrbempty$. Then, it holds that $\ctrb{\argx}{\arga} = \sctrb{\{\argx\}}{\arga}$ (Definition~\ref{principle:ctrb-func-generalization}).
    \begin{description}
        \item[Quantitative contribution existence.] We need to show that if for all partitions $P$ of $\Args \setminus \{\arga\}$ it holds that $\sum_{\argX \in P} \sctrb{\argX}{\arga} = \fs(\arga) - \is(\arga)$ (Principle~\ref{principle:quantitative-ctrb-existence}) then it also holds that $\sum_{\argx \in \Args \setminus \{\arga\}} \ctrb{\argx}{\arga} = \sigma(\arga) - \tau(\arga)$ (Principle~\ref{principle:qexistence}). Because we can assume that $\ctrb{\argx}{\arga} = \sctrb{\{\argx\}}{\arga}$ holds, we can re-write Principle~\ref{principle:quantitative-ctrb-existence} as follows: for the partition $P'$ of $\Args \setminus \{\arga\}$ such that $|P'| = |\Args \setminus \{\arga\}|$ the following statement holds:
        \begin{align*}
            \sum_{\{\argx\} \in P'} \sctrb{\{\argx\}}{\arga} = \fs(\arga) - \is(\arga).
        \end{align*}
        This corresponds to Principle~\ref{principle:quantitative-ctrb-existence}, thus proving the proposition.
        \item[Directionality.] We need to show that if $\sctrb{\argX}{\arga} = 0$ whenever $\forall \argx \in \argX \subseteq \Args \setminus \{\arga\}$ we have that $\argx$ cannot reach $\arga$ (in $\graph$) (Principle~\ref{sprinciple:directionality}) then $\ctrb{\argx}{\arga} = 0$ whenever $\argx$ cannot reach $\arga$ (in $\graph$, Principle~\ref{principle:directionality}). Because $\ctrb{\argx}{\arga} = \sctrb{\{\argx\}}{\arga}$, the following must hold: \\
        \indent if \phantom{ee} $\forall \argx \in \Args \setminus \{\arga\}$ s.t. $\argx$ cannot reach $\arga$ it holds that $\sctrb{\{\argx\}}{\arga} = 0$ \\
        \indent then $\forall \argx \in \Args \setminus \{\arga\}$ s.t. $\argx$ cannot reach $\arga$ it holds that $\ctrb{\argx}{\arga} = 0$. \\
        This proves the proposition.
        
        \item[Counterfactuality.] We need to show that if for any $\argX \subseteq \Args \setminus \{\arga\}$, the following statements hold (Principle~\ref{principle:counterfactuality}):
        \begin{enumerate}[label=(\roman*).a]
             \item If $\sctrb{\argX}{\arga} < 0$, then $\fs_{\graph}(a) < \fs_{\graph\downarrow_{\Args \setminus \argX}}(\arga)$;
             \item If $\sctrb{\argX}{\arga} = 0$, then $\fs_{\graph}(a) = \fs_{\graph\downarrow_{\Args \setminus \argX}}(\arga)$;
             \item If $\sctrb{\argX}{\arga} > 0$, then $\fs_{\graph}(a) > \fs_{\graph\downarrow_{\Args \setminus \argX}}(\arga)$,
        \end{enumerate}
        then the following statements hold for $\argx \in \Args \setminus \{\arga\}$ (Principle~\ref{principle:counterfactual}):
        \begin{enumerate}[label=(\roman*).b]
            \item If $\ctrb{\argx}{\arga} < 0$, then $\fs_\graph(\arga) < \fs_{\graph \downarrow_{\Args \setminus \{\argx\}}}(\arga)$;
            \item If $\ctrb{\argx}{\arga} = 0$, then  $\fs_\graph(\arga) = \fs_{\graph \downarrow_{\Args \setminus \{\argx\}}}(\arga)$;
            \item If $\ctrb{\argx}{\arga} > 0$, then $\fs_\graph(\arga) > \fs_{\graph \downarrow_{\Args \setminus \{\argx\}}}(\arga)$.
        \end{enumerate}
        Because $\ctrb{\argx}{\arga} = \sctrb{\{\argx\}}{\arga}$ holds, we can simply replace $\ctrb{\argx}{\arga}$ by $\sctrb{\{\argx\}}{\arga}$ in \emph{(i).b - (iii).b}.
        Then, we see that \emph{(i).b}, \emph{(ii).b} and \emph{(iii).b} are special cases of \emph{(i).a}, \emph{(ii).a} and \emph{(iii).a}, respectively.
        \item[Quantitative counterfactuality.] We need to show that if for every $\argX \subseteq \Args \setminus \{\arga\}$ it holds that $\sctrb{\argX}{\arga} = \fs_{\graph}(\arga) - \fs_{\graph\downarrow_{\Args \setminus \argX}}(\arga)$ (Principle~\ref{principle:quantitative-counterfactuality}) then it also holds for every $\argx \in \Args \setminus \{\arga\}$ that $\ctrb{\argx}{\arga} = \fs_\graph(\arga) - \fs_{\graph \downarrow_{\Args \setminus \{\argx\}}}(\arga)$ (Principle~\ref{principle:q-counterfactual}). Because $\ctrb{\argx}{\arga} = \sctrb{\{\argx\}}{\arga}$ holds, we can replace $\ctrb{\argx}{\arga}$ in Principle~\ref{principle:q-counterfactual} by $\sctrb{\{\argx\}}{\arga}$. This yields the following statement ($\forall \argx \in \Args \setminus \{\arga\}$):
        \begin{align*}
            \sctrb{\{\argx\}}{\arga} = \fs_\graph(\arga) - \fs_{\graph \downarrow_{\Args \setminus \{\argx\}}}(\arga),
        \end{align*}
        which clearly is a special case of Principle~\ref{principle:quantitative-counterfactuality}, thus proving the proposition.
    \end{description}
\end{proof}
Let us speculate that it may be difficult or impossible to generalize principles in \emph{both} directions at once.
Intuitively, a single-argument contribution function principles such as contribution existence can be $\uparrow$-generalized because expanding the ``search space'' from $\argx \in \Args \setminus \{\arga\}$ to $\argX \subseteq \Args \setminus \{\arga\}$ allows us to find more cases for which the principle's condition holds, and finding \emph{some} (i.e., a single) case given the search space of $\Args$ is sufficient for satisfaction.
In contrast, other principles such as directionality require that a condition holds for \emph{every} $\argx \in \Args \setminus \{\arga\}$ (single-argument case) and $\argX \subseteq \Args \setminus \{\arga\}$ (set case), which makes principle satisfaction more difficult for the latter case.

Now that we have shown that all of the above specified set contribution function principles have single-argument contribution function principles that they either $\uparrow$-generalize or $\downarrow$-generalize, we can move on and define entirely novel set contribution principles that depend on the interactions between several sets and hence are only relevant for the set (and not for the single-argument) case.

What follows are set contribution functions that are not directly based on single-argument contribution counterparts.

\emph{Weak quantitative contribution existence} stipulates that $\Args \setminus \{\arga\}$ must be partitionable such that the sum of the partitions' contributions amounts to the difference between the topic's final and initial strengths.
\begin{principle}[Weak Quantitative Contribution Existence]
\label{principle:weak-quantitative-ctrb-existence}
    A set contribution function $\sctrbempty$ satisfies \emph{weak quantitative contribution existence} w.r.t. a gradual semantics $\fs$ iff there exists a partition $P$ of $\Args \setminus \{\arga\}$ for which it holds that $\sum_{\argX \in P} \sctrb{\argX}{\arga} = \fs(\arga) - \is(\arga)$.
\end{principle}
One may speculate whether \emph{additivity}, stipulating that given disjoint $\argX, \argY \subseteq \Args \setminus \{\arga\}$ it must generally hold that $\sctrbs{\argX}{\arga} + \sctrbs{\argY}{\arga} = \sctrbs{\argX \cup \argY}{\arga}$, is  satisfied as well.
However, if we take the Shapley value-based set contribution function $\sctrbsempty$ and determine the contribution of two disjoint sets $\argX, \argY \subseteq \Args \setminus \{\arga\}$ on a topic, $\sctrbs{\argX}{\arga} + \sctrbs{\argY}{\arga} = \sctrbs{\argX \cup \argY}{\arga}$ may \emph{not} hold.
This is because if it does not hold that $|\argX| = |\argY| = 1$, $\sctrbsempty$ partitions $\Args \setminus \{\arga\}$ differently when computing $\sctrbs{\argX}{\arga}$ than when computing $\sctrbs{\argY}{\arga}$ and---independently of the aforementioned condition---yet again differently when computing $\sctrbs{\argX \cup \argY}{\arga}$.
This means that, from the perspective of Shapley values, we actually compare the results of different coalitional games.
Counterexamples can be easily found and are provided in the notebook (code) that contains all examples; for the sake of brevity, we omit them from the paper.

As the below principles describe function behavior given two set contributors, we (again) assume $\argX, \argY \subseteq \Args \setminus \{\arga\}$.
\emph{Consistency} stipulates that if the signs of two set contributions are the same, then the contribution of the union of the corresponding set contributors should also have the same sign as the two separate contributions.
\begin{principle}[Consistency]
\label{principle:consistency}
    A set contribution function $\sctrbempty$ satisfies \emph{consistency} w.r.t. a gradual semantics $\fs$ iff the following statements hold true:
    \begin{itemize}
        \item If $\sctrb{\argX}{\arga} \leq 0$ and $\sctrb{\argY}{\arga} \leq 0$ then $\sctrb{\argX \cup \argY}{\arga} \leq 0$;
        \item If $\sctrb{\argX}{\arga} \geq 0$ and $\sctrb{\argY}{\arga} \geq 0$ then $\sctrb{\argX \cup \argY}{\arga} \geq 0$.
    \end{itemize}
\end{principle}
\emph{Monotonicity} requires that the contribution of a set contributor increases weakly monotonically when adding arguments to the set.
\begin{principle}[Monotonicity]
\label{principle:monotonicity}
     A set contribution function $\sctrbempty$ satisfies \emph{monotonicity} w.r.t. a gradual semantics $\fs$ iff $\argX \subseteq \argY$ implies that $\sctrb{\argX}{\arga} \leq \sctrb{\argY}{\arga}$.
\end{principle}
\emph{Monotonicity} follows the idea that increasing the set of contributors increases our options with which we can affect change and thus, given a maximization objective, should (weakly) increase the contribution.
Alternatively, in case the objective is to minimize the topic's final strength, one might as well desire to push change in the negative direction, for which the direction of the comparison would need to be inverted.

Finally, \emph{symmetry} requires that if a topic's supporters and attackers are equally strong (with respect to a one-to-one mapping of final strengths), then the contribution of the union of all supporters and attackers to the topic is $0$.
As a prerequisite, given $a \in \Args$, we denote by ${\cal R}^-(a)$ the set of all attackers of $a$ and by ${\cal R}^+(a)$ the set of all supporters of $a$.

\begin{principle}[Symmetry]
\label{principle:symmetry} 
    A set contribution function $\sctrbempty$ satisfies symmetry whenever it holds that if the multisets $\{\fs(\argx) | \argx \in {\cal R}^-(\arga)\}$ and $\{\fs(\argx) | \argx \in {\cal R}^+(\arga)\}$ are equal
    then $\sctrb{{\cal R}^-(\arga) \cup {\cal R}^+(\arga)}{\arga} = 0$.
\end{principle}

\section{Principle-based Analysis}
\label{sec:analysis}
We proceed by providing a comprehensive analysis of principle satisfaction and violation for our four set contribution functions, covering all principles from the previous sections, as well as the five semantics---QE, DFQuAD, SD-DFQuAD, EB, and EBT semantics---from Table~\ref{table:semanticsExamples}.
Recall that Tables~\ref{table:principle-overview-1} and~\ref{table:principle-overview-2} provide an overview of the results.
The code used to compute exact values for all counterexamples in this analysis is available at \url{https://github.com/TimKam/Quantitative-Bipolar-Argumentation/tree/main/analysis}.

\subsection{(Quantitative) Contribution Existence}
\label{subsec:contribution-existence}

$\sctrbrempty$, $\sctrbriempty$, and $\sctrbsempty$ all satisfy contribution existence w.r.t. all of the surveyed gradual semantics.
These contributions are all based on some sort of removal-based modification of a given QBAG.
Intuitively, we can always remove all arguments that can reach a topic, excluding the topic itself.
This set of contributors then fully accounts for the difference between the topic's initial and final strengths.
\begin{proposition}\label{prop:pre-contrb-existence-positive}
    For every $\sctrbempty \in \{\sctrbrempty, \sctrbriempty, \sctrbsempty\}$, for every gradual semantics $\fs$ among QE, DFQuAD, SD-DFQuAD, EB, and EBT semantics it holds that $\exists \argX \subseteq \Args \setminus \{\arga\}$ s.t. $\sctrb{\argX}{\arga} = \fs(\arga) - \is(\arga)$.
\end{proposition}
\begin{proof}
    Consider the set $\argX \subseteq \Args \setminus \{ \arga\}$ that contains exactly the arguments that can reach $\arga$ except $\arga$ itself. Given any of our semantics $\fs$ it clearly holds that 
    (consider \emph{bi-variate independence} and \emph{bi-variate directionality}, i.e., Semantics Principles~\ref{sprinciple:bi-variate-independence} and~\ref{sprinciple:bi-variate-directionality}):
    \begin{enumerate}[label=\roman*)]
        \item \textbf{Removal-based.} $\fs_{\graph\downarrow_{\Args \setminus \argX}}(\arga) = \is(\arga)$ and hence $\sctrbr{\argX}{\arga} = \fs(\arga) - \is(\arga)$;
        \item \textbf{Intrinsic removal-based.} $(\Args, \is, \Att \setminus \{(\argy, \argx) \in \Att | \argy \notin \argX, \argx  \in \argX \}, \Supp \setminus \{(\argy, \argx) \in \Supp | \argy \notin \argX, \argx  \in \argX \}) = \graph$ and hence it follows from \emph{i)} that $\sctrbri{\argX}{\arga} = \fs(\arga) - \is(\arga)$;
        \item \textbf{Shapley value-based.} $\forall S \subseteq \Args \setminus (\argX \cup \{\arga\})$ it holds that $\fs_{\graph\downarrow_{\Args \setminus S}}(\arga) = \fs(\arga)$ and $\fs_{\graph\downarrow_{\Args \setminus (S \cup \argX)}}(\arga) = \is(\arga)$, and hence $\sctrbs{\argX}{\arga} = \fs(\arga) - \is(\arga)$.
    \end{enumerate}
\end{proof}
Note that if we constrained our set contributor $\argX$ to have a cardinality of $1$, we could no longer satisfy this property in some cases.
Notably, when applying the \emph{Top} aggregation function or $p$-Max$(k)$ influence functions in modular semantics, removing any single argument may not have any effect on the final strength of a topic.
This explains why, in the case of set contributions, contribution existence is satisfied for more functions/semantics than for single-argument contributions (compare Table~\ref{table:principle-overview-1} with Table 1 in \cite{DBLP:journals/ijar/KampikPYCT24}).
\begin{corollary}\label{cor:contrb-existence-positive}
    $\sctrbrempty$, $\sctrbriempty$, and $\sctrbsempty$ satisfy the contribution existence principle w.r.t. QE, DFQuAD, SD-DFQuAD, EB, and EBT semantics $\fs$.
\end{corollary}
\begin{proof}
    The proof follows directly from Proposition~\ref{prop:pre-contrb-existence-positive}: for $\sctrbempty \in \{\sctrbrempty, \sctrbriempty, \sctrbsempty\}$ there must exist an $\argX \subseteq \Args \setminus \{\arga\}$ s.t. $\sctrb{\argX}{\arga} = \fs(\arga) - \is(\arga)$ and hence, if a topic's initial strength is not equal its final strength we have $\sctrb{\argX}{\arga} \neq 0$, thus satisfying contribution existence.
\end{proof}
$\sctrbgmempty$ satisfies the contribution existence principle with respect to just \emph{some} of our semantics, specifically the ones for which $\ctrbgempty$ satisfies the single-argument version of the principle: QE and EB semantics~\cite{DBLP:journals/ijar/KampikPYCT24}.
\begin{proposition}\label{prop:contribution-existence-positive}
    $\sctrbgmempty$ satisfies the contribution existence principle w.r.t. QE and EB semantics $\fs$.
\end{proposition}
\begin{proof}
    Clearly, $\sctrbgmempty$ satisfies (set) contribution existence w.r.t. a gradual semantics $\fs$ if $\ctrbgempty$ satisfies (single-argument) contribution existence w.r.t. $\fs$: if there exists an argument $\argx \in \Args \setminus \{\arga\}$ such that $\ctrbg{\argx}{\arga} \neq 0$ then $\sctrbgm{\{\argx\}}{\arga} \neq 0$. As $\ctrbgempty$ indeed satisfies contribution existence w.r.t. QE and EB (Proposition 5.7 in~\cite{DBLP:journals/ijar/KampikPYCT24}; note that QE's and EB's satisfaction of 
    \emph{bi-variate independence} and \emph{bi-variate directionality}---Semantics Principles~\ref{sprinciple:bi-variate-independence} and \ref{sprinciple:bi-variate-directionality}---is crucial), the proposition must hold. 
\end{proof}
For the other semantics, i.e., DFQuAD, SD-DFQuAD, and EBT, contribution existence is violated by $\sctrbgmempty$.
Intuitively, the fact that we aggregate the gradient of single arguments leads to analogous issues as for single-argument contribution functions, notably in some cases where we have several arguments with maximum final strengths attacking our topic, as seen in Figure~\ref{fig:proof-contribution-existence}. 
\begin{proposition}\label{prop:contribution-existence}
     $\sctrbgmempty$ violates the contribution existence principle w.r.t. DFQuAD, SD-DFQuAD, and EBT semantics $\fs$.
\end{proposition}
\begin{proof}
    Consider the QBAG displayed in \autoref{fig:proof-contribution-existence}, 
    which we denote by $\graph = (\Args, \is, \Att, \Supp)$. Given DFQuAD, SD-DFQuAD, and EBT semantics $\fs$, 
    we have $\sctrbgm{\{\argb\}}{\arga} = \sctrbgm{\{\argc\}}{\arga} = 0$. Hence, it should hold that $\fs_{\graph}(\arga) = \is(\arga)$. 
    However, we observe that $\fs_{\graph}(\arga) < \is(\arga) = 0.5$ which proves the violation of the contribution existence principle for $\sctrbgmempty$.
\end{proof}
\begin{figure}[h!]
    \centering
    \begin{tikzpicture}[node distance=1.5cm]
        \node[unanode] (a) at (0, 0) {\argnode{\arga}{0.5}{< 0.5}};
        \node[unanode] (b) at (-2, 0) {\argnode{\argb}{1.0}{1.0}};
        \node[unanode] (c) at (2, 0) {\argnode{\argc}{1.0}{1.0}};

        \draw[-stealth, thick] (b) -- node[pos=0.5, above=1pt] {-} (a);
        \draw[-stealth, thick] (c) -- node[pos=0.5, above=1pt] {-} (a);
    
    \end{tikzpicture}
    \caption{$\sctrbgmempty$ violates the contribution existence principle w.r.t. DFQuAD, SD-DFQuAD, and EBT semantics $\fs$.}
    \label{fig:proof-contribution-existence}
\end{figure}
While contribution existence is satisfied by most set contribution functions with respect to most semantics, \emph{quantitative} contribution existence is violated.
Figure~\ref{fig:proof-contribution-existence} provides the counterexample for $\sctrbrempty, \sctrbriempty$, and $\sctrbgmempty$ (with respect to all of the studied semantics).
\begin{proposition}\label{prop:quantitative-contribution-existence}
    $\sctrbrempty, \sctrbriempty$, and $\sctrbgmempty$ violate the quantitative contribution existence principle w.r.t. QE, DFQuAD, SD-DFQuAD, EB, and EBT semantics $\fs$.
\end{proposition}
\begin{proof}
    Consider the QBAG displayed in \autoref{fig:proof-contribution-existence}, which we denote by $\graph = (\Args, \is, \Att, \Supp)$. 
    Given QE, DFQuAD, SD-DFQuAD, EB, and EBT semantics $\fs$ we find that for all $\sctrbempty \in \{\sctrbrempty, \sctrbriempty, \sctrbgmempty\}$, it holds that $\sctrb{\{\argb\}}{\arga} + \sctrb{\{\argc\}}{\arga} \neq \fs_{\graph}(\arga) - \is(\arga)$. 
    This proves the violation of quantitative contribution existence for these set contribution functions.
\end{proof}
For instance, using $\sctrbrempty$ under the QE semantics we get $\sctrbr{\{\argb\}}{\arga} = -0.15$, $\sctrb{\{\argc\}}{\arga} = -0.15$, and $\fs_{\graph}(\arga) - \is(\arga) = -0.4$.

An alternative to the above proof by counterexample could be provided by observing that quantitative contribution existence for set contribution functions $\downarrow$-generalizes the corresponding single argument contribution function principle (Proposition~\ref{prop:down-generalization}). As the three single-argument contribution functions violate quantitative contribution existence (Table 1 in~\cite{DBLP:journals/ijar/KampikPYCT24}), this must then also be the case for their $\downarrow$-generalizations.
A somewhat analogous proof is provided for Proposition~\ref{prop:neg-all-counterfactuality}.

$\sctrbsempty$ violates quantitative contribution existence as well.
Intuitively, one may expect satisfaction, as the principle closely resembles the \emph{efficiency} property of the Shapley value.
However, when applying $\sctrbsempty$ to compute set contributions of two disjoint set contributors of the same QBAG, this may lead to the formation of different coalitional games, from a Shapley value perspective.
For example, given a set of arguments $\{\arga, \argb, \argc, \argd\}$, computing the set contribution of  $\{\arga, \argb\}$ to $\argd$ forms a game with players $\{\arga, \argb\}$ and $\{\argc\}$, whereas computing the set contribution of $\{\argc\}$ to $\argd$ forms a game with the players $\{\arga\}$, $\{\argb\}$, and $\{\argc\}$. Across different games, the property does not hold; i.e., quantitative contribution existence is then no longer guaranteed.
\begin{proposition}\label{prop:quantitative-contribution-existence-shapley}
    $\sctrbsempty$ violates the quantitative contribution existence principle w.r.t. QE, DFQuAD, SD-DFQuAD, EB, and EBT semantics $\fs$.
\end{proposition}
\begin{proof}
     Consider the QBAG displayed in \autoref{fig:proof-quantitative-contribution-existence-shapley}, which we denote by $\graph = (\Args, \is, \Att, \Supp)$. 
    Given QE, DFQuAD, SD-DFQuAD, EB, and EBT semantics $\fs$ we observe that $\sctrb{\{\argb, \argc\}}{\arga} + \sctrb{\{\argd\}}{\arga} \neq \is(\arga) - \fs_{\graph}(\arga)$. The definition of the quantitative contribution existence principle states that for all partitions $P$ of $\Args \setminus \{\arga\}$ it must hold that $\sum_{\argX \in P} \sctrb{\argX}{\arga} = \fs(\arga) - \is(\arga)$. Hence, observing a single partition where this is not the case proves the violation of quantitative contribution existence for $\sctrbsempty$.
\end{proof}
\begin{figure}[h!]
    \centering
    \begin{tikzpicture}[node distance=1.5cm]
        \node[unanode] (a) at (0, 0) {\argnode{\arga}{0.5}{< 0.5}};
        \node[unanode] (b) at (-2, 0) {\argnode{\argb}{0.5}{0.5}};
        \node[unanode] (c) at (2, 0) {\argnode{\argc}{0.5}{< 0.5}};
        \node[unanode] (d) at (4, 0) {\argnode{\argd}{0.5}{0.5}};

        \draw[-stealth, thick] (b) -- node[pos=0.5, above=1pt] {-} (a);
        \draw[-stealth, thick] (c) -- node[pos=0.5, above=1pt] {-} (a);
        \draw[-stealth, thick] (d) -- node[pos=0.5, above=1pt] {-} (c);
    
    \end{tikzpicture}
    \caption{$\sctrbsempty$ violates the quantitative contribution existence principle w.r.t. QE, DFQuAD, SD-DFQuAD, EB, and EBT semantics $\fs$.}
    \label{fig:proof-quantitative-contribution-existence-shapley}
\end{figure}
%

\subsection{Directionality}
Intuitively, all of our set contribution functions satisfy directionality given any of the surveyed gradual semantics.
This is because given a modular semantics, the final strength of a (topic) argument depends only on its initial strength and the final strengths of its attackers and supporters---a property that is sometimes referred to as \emph{stability}~\cite{DBLP:conf/kr/Potyka18,DBLP:journals/ijar/KampikPYCT24}.
All of the introduced set contribution functions determine the set contribution to a topic by either removing contributors or changing their initial strengths.
This has no effect if the contributors cannot reach the topic.
\begin{proposition}\label{prop:directionality}
    $\sctrbrempty$, $\sctrbriempty$, $\sctrbsempty$, and $\sctrbgmempty$ satisfy the directionality principle w.r.t. QE, DFQuAD, SD-DFQuAD, EB, and EBT semantics $\fs$.
\end{proposition}
\begin{proof}
    Observe that given a gradual semantics $\fs$ that is either QE, DFQuAD, SD-DFQuAD, EB, or EBT semantics, for every $\argX \subseteq \Args$ s.t. $\forall \argx \in \argX$ we have that $\argx$ cannot reach $\arga \in \Args$ the following holds:
    \begin{enumerate}[label=\roman*)]
        \item $\fs_\graph(\arga) = \fs_{\graph\downarrow_{\Args \setminus \argX}}(\arga)$;
        \item $\fs_\graph(\arga) = \fs_{(\Args, \is, \Att \setminus \{(\argy, \argx) \in \Att | \argy \notin \argX, \argx  \in \argX \}), \Supp \setminus \{(\argy, \argx) \in \Supp | \argy \notin \argX, \argx  \in \argX \})}(\arga)$;
        \item $\forall S \subseteq \Args \setminus (\argX \cup \{\arga\})$ it holds that $\operatorname{\fs}_{\graph\downarrow_{\Args \setminus S}}(\arga) = \operatorname{\fs}_{\graph\downarrow_{\Args \setminus (S \cup \argX)}}(\arga)$;
        \item $\forall \argx \in \argX$, for $\epsilon \in \interval$ it holds that $\fs_\graph(\arga) = \fs_{\graph\downarrow_{\tau(\argx) \leftarrow \epsilon}}(\arga)$.
    \end{enumerate}
    From \emph{i)} it follows that $\sctrbrempty$ satisfies directionality; from \emph{i)} and \emph{ii)} it follows that $\sctrbriempty$ satisfies directionality; from \emph{iii)} it follows that $\sctrbsempty$ satisfies directionality;  from \emph{iv)} it follows that $\sctrbgmempty$ satisfies directionality (in all cases w.r.t. all of the above semantics).
\end{proof}
%

\subsection{(Quantitative) Counterfactuality}
Quantitative counterfactuality and counterfactuality are trivially satisfied for the removal-based set contribution function.
Its contributions amount to what is exactly expected given \emph{quantitative counterfactuality}, and \emph{counterfactuality} is a weaker case of the former principle.
\begin{proposition}\label{prop:remove-counterfactuality}
    $\sctrbrempty$ satisfies the counterfactuality and quantitative counterfactuality principles w.r.t. QE, DFQuAD, and SD-DFQuAD, EB, and EBT semantics $\fs$.
\end{proposition}
\begin{proof}
    Quantitative counterfactuality requires that $\sctrbr{\argX}{\arga} = \fs_{\graph}(\arga) - \fs_{\graph\downarrow_{\Args \setminus \argX}}(\arga)$. This holds by definition (Definition~\ref{def:set-removal}), which stipulates that $\sctrbr{\argX}{\arga} = \fs_\graph(\arga) - \fs_{\graph\downarrow_{\Args \setminus \argX}}(\arga)$.
    Counterfactuality is a weaker principle, requiring that the sign of $\sctrbr{\argX}{\arga}$ must be consistent with the sign of $\fs_\graph(\arga) - \fs_{\graph\downarrow_{\Args \setminus \argX}}(\arga)$ (incl. for the case of a $0$-contribution, i.e., in the case of $\sctrbr{\argX}{\arga} = 0$ whenever $\fs_\graph(\arga) = \fs_{\graph\downarrow_{\Args \setminus \argX}}(\arga)$), which again is obviously the case because $\sctrbr{\argX}{\arga} = \fs_\graph(\arga) - \fs_{\graph\downarrow_{\Args \setminus \argX}}(\arga)$.
\end{proof}
All other set contribution functions violate counterfactuality and quantitative counterfactuality w.r.t. the surveyed semantics: this must be the case because the two principles $\downarrow$-generalize the corresponding single-argument contribution function principles and the latter principles are violated by the single-argument contribution function counterparts of intrinsic removal-based, Shapley value-based, and gradient-based set contribution functions, as observed in~\cite{DBLP:journals/ijar/KampikPYCT24}.
\begin{proposition}\label{prop:neg-all-counterfactuality}
    $\sctrbriempty$, $\sctrbsempty$, and $\sctrbgmempty$ violate the counterfactuality and quantitative counterfactuality principles w.r.t. QE, DFQuAD, SD-DFQuAD, EB, and EBT semantics $\fs$.
\end{proposition}
\begin{proof}
    We have shown that $\sctrbriempty$, $\sctrbgmempty$ and $\sctrbsempty$ generalize $\ctrbriempty$, $\ctrbgempty$ and $\ctrbsempty$, respectively (Proposition~\ref{prop:ctrb-function-generalization}) and that the set contribution function principles \emph{counterfactuality} and \emph{quantitative counterfactuality} $\downarrow$-generalize the respective single-argument contribution function principles (Proposition~\ref{prop:down-generalization}). This means that if any of the set contribution functions satisfies one of these principles, the corresponding single-argument contribution functions must satisfy the (single-argument version of the) principle as well. However, $\ctrbriempty$, $\ctrbsempty$, and $\ctrbgempty$ violate both the counterfactuality and quantitative counterfactuality principles w.r.t. QE, DFQuAD, SD-DFQuAD, EB, and EBT semantics $\fs$~\cite{DBLP:journals/ijar/KampikPYCT24} (Table 1, analysis in Subsection 5.4). This means that $\sctrbriempty$, $\sctrbsempty$, and $\sctrbgmempty$ cannot possibly satisfy the counterfactuality and quantitative counterfactuality principles, as this would then contradict the above observation.
\end{proof}
Indeed, we claim that quantitative counterfactuality is uniquely satisfied only by the removal-based set contribution function, as the desired behavior as specified by the principle amounts to the function's entire behavior.
In the appendix, we provide further counterexamples that may help readers gain better intuitions for cases when the principles are violated.

Intuitively, the intrinsic removal-based set contribution functions violates counterfactuality and quantitative counterfactuality because it controls for ``external'' influence on a set contributor, which leads to a misalignment with the idea underlying the principles. Similarly, Shapley violates counterfactuality and quantitative counterfactuality because it assigns each player (i.e., single argument or set of arguments) its average marginal contribution over all coalitions, so any effect realized with others gets split across the coalition, and the contributor's value need not match the removal effect. Gradient violates the principles because it measures \emph{sensitivity} at the current point, not the effect of deleting, which again leads to misalignment with the idea underlying the principles.

\subsection{Weak Quantitative Contribution Existence}
\label{subsec:weak-quantitative-ctrb-ex}
In Subsection~\ref{subsec:contribution-existence}, we have shown that quantitative contribution existence is violated by all four set contribution functions with respect to all five semantics.
However, the \emph{weak} quantitative contribution existence principle is satisfied by $\sctrbrempty$, $\sctrbriempty$, and $\sctrbsempty$; intuitively, we can always take as set contributor the set of all arguments that can reach the topic (excluding the topic itself), whose contribution then amounts to the difference between the topic's final and initial strengths, whereas all other arguments' contributions are zero.
\begin{proposition}\label{prop:weak-contrb-existence-positive}
    $\sctrbrempty$, $\sctrbriempty$, and $\sctrbsempty$ satisfy the weak quantitative contribution existence principle w.r.t. QE, DFQuAD, SD-DFQuAD, EB, and EBT semantics $\fs$.
\end{proposition}
\begin{proof}
    Consider a set contribution function $\sctrbempty$ among $\sctrbrempty$, $\sctrbriempty$, and $\sctrbsempty$ and a gradual semantics $\fs$ among QE, DFQuAD, SD-DFQuAD, EB, and EBT semantics. From Proposition~\ref{prop:pre-contrb-existence-positive}, we know that  $\exists \argX \subseteq \Args \setminus \{\arga\}$ s.t. $\sctrb{\argX}{\arga} = \fs(\arga) - \is(\arga)$; in the proof of Proposition~\ref{prop:pre-contrb-existence-positive}, we have observed that this holds for $\argX$ when $\argX$ is the set of all arguments that can reach $\arga$ (except $\arga$ itself). Because $\sctrbempty$ satisfies directionality w.r.t. $\fs$ (Proposition~\ref{prop:directionality}), we know that for $\argY = \Args \setminus (\argX \cup \{\arga\})$ it must hold that $\sctrb{\argY}{\arga} = 0$.  
    Accordingly, we have partitioned $\Args \setminus \{\arga\}$ into $P = \{\argX, \argY\}$ when $\argY \neq \emptyset$.
    If $\argY = \emptyset$ (equivalently, $\argX = \Args \setminus \{\arga\}$) consider instead the partition $P = \{\argX\}$. 
    For $P$ it holds that $\sum_{\argX' \in P} \sctrb{\argX'}{\arga} = \fs(\arga) - \is(\arga)$, which proves the satisfaction of weak quantitative contribution existence (Principle~\ref{principle:weak-quantitative-ctrb-existence}).
\end{proof}
In contrast, $\sctrbgmempty$ violates even the weak quantitative contribution existence principle.
Intuitively, the contribution function quantifies the effect and direction of the maximal marginal change to a contributor's initial strength on the topic's final strength;
this is incompatible with the idea of fully accounting for the difference between a topic's final and initial strengths.
Note that allowing the partition $P = \{Args \setminus \{\arga\}\}$ means that every of our set contribution functions except $\sctrbgmempty$ can trivially satisfy weak quantitative contribution existence. 
Interestingly, a principle that meaningfully discriminates between contribution functions in the single-argument setting~\cite{DBLP:journals/ijar/KampikPYCT24} loses its power when generalized to sets.
\begin{proposition}\label{prop:weak-contrb-existence-negative}
    $\sctrbgmempty$ violates the weak contribution existence principle w.r.t. QE, DFQuAD, SD-DFQuAD, EB, and EBT semantics $\fs$.
\end{proposition}
\begin{proof}
    Consider the QBAG displayed in \autoref{fig:proof-gradient-weak-contribution-existence}.
    According to the definition of the weak contribution existence principle, the total contribution of at least one of the argument partitions $P$ of $\Args \setminus \{\arga\}$, 
    i.e., $[\{\argb\}, \{\argc\}]$ and $[\{\argb,\argc\}]$ needs to equal the difference in initial strength and final strength, $\fs(\arga) - \is(\arga)$, of the topic argument $\arga$.
    However, this is not the case, proving the violation of the principle.
\end{proof}
\begin{figure}[!ht]
    \centering
    \begin{tikzpicture}[node distance=1.5cm]
        \node[unanode] (c) at (-2, 0) {\argnode{\argc}{1.0}{1.0}};
        \node[unanode] (b) at (2, 0) {\argnode{\argb}{1.0}{1.0}};
        \node[unanode] (a) at (0, 0) {\argnode{\arga}{0.5}{< 0.5}};
        
        \draw[-stealth, thick] (c) -- node[pos=0.5, above=1pt] {-} (a);
        \draw[-stealth, thick] (b) -- node[pos=0.5, above=1pt] {-} (a);
    \end{tikzpicture}
    \caption{$\sctrbgmempty$ violates the weak contribution existence principle w.r.t. QE, DFQuAD, SD-DFQuAD, EB, and EBT semantics $\fs$.}
    \label{fig:proof-gradient-weak-contribution-existence}
\end{figure}
%

\subsection{Consistency}
\label{subsec:consistency}
Intuitively, one may expect that consistency is satisfied by all four contribution functions, given the five semantics.
After all, considering more positive/negative contributors in a set contributor should always lead to an increased/decreased contribution.
However, this is not the case.
As observed in~\cite{DBLP:journals/ijar/KampikPYCT24}, one argument's initial strength may have a somewhat \emph{nonmonotonic} effect on the final strength of an affected argument. 
E.g., given an initial strength value of $0.3$ of an argument $\argx$, marginally increasing this value leads to a positive change in argument $\argy$'s final strength, but given an initial strength value of $0.6$, such marginal increase leads to a negative change.
Similarly, the effects of two arguments may interact, as we can see in Figure~\ref{fig:proof-consistency} (also consider the sign map in Figure~\ref{fig:sign-map}).
\begin{proposition}\label{prop:remove-consistency}
    $\sctrbrempty$ and $\sctrbriempty$ violate the consistency principle w.r.t. QE, DFQuAD, SD-DFQuAD, EB, EBT semantics $\fs$.
\end{proposition}
\begin{proof}
    Consider the QBAG displayed in \autoref{fig:proof-consistency}, which we denote by $\graph = (\Args, \is, \Att, \Supp)$. Consider the following sets of initial strength values:
    \begin{itemize}
        \item $\is_{QE} = \{(\arga, 0.3), (b, 0.8), (c, 0.1), (d, 0.55), (e, 0.45), (f, 0.6)\}$,
        \item $\is_{DFQuAD} = \{(a, 0.3), (b, 0.8), (c, 0.1), (d, 0.8), (e, 0.6), (f, 0.8)\}$,
        \item $\is_{SD-DFQuAD} = \{(a, 0.3), (b, 0.8), (c, 0.1), (d, 0.25), (e, 0.2), (f, 0.2)\}$,
        \item $\is_{EB} = \{(a, 0.3), (b, 0.8), (c, 0.9), (d, 0.9), (e, 0.9), (f, 0.8)\}$,
        \item $\is_{EBT} = \{(a, 0.3), (b, 0.1), (c, 0.1), (d, 0.5), (e, 0.1), (f, 0.5)\}$.
    \end{itemize}
    
    Let $\arga$ be the topic argument. By definition of the consistency principle, if both $\{\argx\}$ and $\{\argy\}$ contribute negatively to the topic argument's final strength, then $\{\argx\} \cup \{\argy\}$ must also contribute negatively. However, computing $\fs_{\arga}$ on $\graph$ for $\sctrbempty \in \{\sctrbrempty, \sctrbriempty\}$ with the given initial strength values $\is_{QE}, \is_{DFQuAD}, \is_{SD-DFQuAD}, \is_{EB}$, and $\is_{EBT}$ yields situations where both $\sctrb{\{\argd\}}{\arga}$ and $\sctrb{\{\argf\}}{\arga}$ are non-positive (or non-negative), while their joint set contribution $\sctrb{\{\argd,\argf\}}{\arga}$ is strictly of the opposite sign. Hence, two single argument contributions that agree in direction combine into a set contribution that points in the other direction, thus contradicting the principle. 
\end{proof}
\begin{figure}[!ht]
    \centering
    \begin{tikzpicture}[node distance=1.5cm]
            \node[unanode] (d) at (-2, 2) {\argnode{\argd}{\is(d)}{\fs(d)}};
            \node[unanode] (f) at (2, 2) {\argnode{\argf}{\is(f)}{\fs(f)}};
            \node[unanode] (e) at (2, 0) {\argnode{\arge}{\is(e)}{\fs(e)}};
            \node[unanode] (c) at (0, 0) {\argnode{\argc}{\is(c)}{\fs(c)}};
            \node[unanode] (b) at (-2, 0) {\argnode{\argb}{\is(b)}{\fs(b)}};
            \node[unanode] (a) at (0, -2) {\argnode{\arga}{\is(a)}{\fs(a)}};
            
            \draw[-stealth, thick] (d) -- node[pos=0.5, left=1pt] {+} (c);
            \draw[-stealth, thick] (d) -- node[pos=0.5, left=1pt] {+} (e);
            \draw[-stealth, thick] (d) -- node[pos=0.5, left=1pt] {+} (b);
    
            \draw[-stealth, thick] (f) -- node[pos=0.5, right=1pt] {+} (c);
            \draw[-stealth, thick] (f) -- node[pos=0.5, right=1pt] {+} (e);
            \draw[-stealth, thick] (f) -- node[pos=0.5, right=1pt] {+} (b);
    
            \draw[-stealth, thick] (c) -- node[pos=0.5, below=1pt] {-} (e);
            \draw[-stealth, thick] (c) -- node[pos=0.5, below=1pt] {+} (b);
            
            \draw[-stealth, thick] (e) -- node[pos=0.5, right=1pt] {-} (a);
            
            \draw[-stealth, thick] (b) -- node[pos=0.5, left=1pt] {+} (a);

        \end{tikzpicture}
    \caption{$\sctrbrempty$, $\sctrbriempty$, and $\sctrbsempty$ violate the consistency principle w.r.t. QE, DFQuAD, SD-DFQuAD, EB, and EBT semantics $\fs$.}
    \label{fig:proof-consistency}
\end{figure}

A similar counter-example can be provided for the Shapley value-based set contribution function.
\begin{proposition}\label{prop:shapley-consistency}
    $\sctrbsempty$ violates the consistency principle w.r.t. QE, DFQuAD, SD-DFQuAD, EB, EBT semantics $\fs$.
\end{proposition}
\begin{proof}
    Consider the QBAG displayed in \autoref{fig:proof-consistency}, which we denote by $\graph = (\Args, \is, \Att, \Supp)$. Consider the following sets of initial strength values:
    \begin{itemize}
        \item $\is_{QE} = \{(a, 0.3), (b, 0.8), (c, 0.2), (d, 0.7), (e, 0.6), (f, 0.7)\}$,
        \item $\is_{DFQuAD} = \{(a, 0.3), (b, 0.8), (c, 0.2), (d, 0.5), (e, 0.4), (f, 0.5)\}$,
        \item $\is_{SD-DFQuAD} = \{(a, 0.3), (b, 0.1), (c, 0.9), (d, 0.3), (e, 0.9), (f, 0.3)\}$,
        \item $\is_{EB} = \{(a, 0.3), (b, 0.6), (c, 0.8), (d, 0.8), (e, 0.5), (f, 0.7)\}$,
        \item $\is_{EBT} = \{(a, 0.3), (b, 0.6), (c, 0.2), (d, 0.8), (e, 0.4), (f, 0.7)\}$.
    \end{itemize}
    
    Let $\arga$ be the topic argument. By definition of the consistency principle, if both $\{\argx\}$ and $\{\argy\}$ contribute negatively to the topic argument's final strength, then $\{\argx\} \cup \{\argy\}$ must also contribute negatively. However, computing $\fs_{\arga}$ on $\graph$ with the given initial strength values $\is_{QE}, \is_{DFQuAD}, \is_{SD-DFQuAD}, \is_{EB}$, and $\is_{EBT}$ yields situations where both $\sctrb{\{\argd\}}{\arga}$ and $\sctrb{\{\argf\}}{\arga}$ are non-positive (or non-negative), while their joint set contribution $\sctrb{\{\argd,\argf\}}{\arga}$ is strictly of the opposite sign. Hence, two single argument contributions that agree in direction combine into a set contribution that points in the other direction, thus contradicting the principle.
\end{proof}
We may speculate that there could be special cases for which consistency is not violated, given QBAGs in which the ``effect'' (e.g., in the sense of the removal-based or gradient-based contribution) of a given single argument to a topic always has the same sign, irrespective of initial strength assignments; we leave a more precise investigation of this issue for future work.

As $\sctrbgmempty$ takes the maximum of single-argument contributions, it satisfies the consistency principle.
\begin{proposition}\label{prop:consistency-positive}
    $\sctrbgmempty$ satisfies the consistency principle w.r.t. QE, DFQuAD, SD-DFQuAD, EB, and EBT semantics $\fs$.
\end{proposition}
\begin{proof}
     $\sctrbgmempty$ is defined as the maximal gradient of a single argument in a set of contributors. For the sake of simplicity, let us denote this by $\sctrbgm{\argX}{\arga} := \max_{\argx \in \argX} g(\argx)$.
     Clearly, the following holds:
     \begin{itemize}
        \item If $\max_{\argx \in \argX} g(\argx) \leq 0$ and $\max_{\argy \in \argY} g(\argy) \leq 0$ then $\max_{\argx \in \argX \cup \argY} g(\argx) \leq 0$ (proving that the first condition of Principle~\ref{principle:consistency} holds);
        \item If $\max_{\argx \in \argX} g(\argx) \geq 0$ and $\max_{\argy \in \argY} g(\argy) \geq 0$ then $\max_{\argx \in \argX \cup \argY} g(\argx) \geq 0$ (proving that the second condition of Principle~\ref{principle:consistency} holds).
    \end{itemize}
\end{proof}

\subsection{Monotonicity}
\label{subsec:monotonicity}
$\sctrbrempty, \sctrbriempty$, and $\sctrbsempty$ clearly violate monotonicity.
Intuitively, when the set contributor is expanded with arguments whose presence has a negative effect on the topic's final strength, this can lead to a decrease of the set contribution.
\begin{proposition}\label{prop:monotonicity-negative}
    $\sctrbrempty, \sctrbriempty$, and $\sctrbsempty$ violate the monotonicity principle w.r.t. QE, DFQuAD, SD-DFQuAD, EB, and EBT semantics $\fs$.
\end{proposition}
\begin{proof}
    Consider the QBAG displayed in \autoref{fig:proof-monotonicity-negative}. 
    Given QE, DFQuAD, SD-DFQuAD, EB, and EBT semantics $\fs$, we find that for all $\sctrbempty \in \{\sctrbrempty, \sctrbriempty, \sctrbsempty\}$, it holds that $\sctrb{\{\argc, \argb\}}{\arga} = 0$. Hence, according to the definition of the monotonicity principle, we would expect $\sctrb{\{\argc\}}{\arga} \leq 0$. However, we observe that $\sctrb{\{\argc\}}{\arga} > 0$, which proves the violation of the monotonicity principle. 
\end{proof}
\begin{figure}[!ht]
    \centering
    \begin{tikzpicture}[node distance=1.5cm]
        \node[unanode] (b) at (-2, 0) {\argnode{\argb}{1.0}{1.0}};
        \node[unanode] (c) at (2, 0) {\argnode{\argc}{1.0}{1.0}};
        \node[unanode] (a) at (0, 0) {\argnode{\arga}{0.5}{0.5}};
        
        \draw[-stealth, thick] (b) -- node[pos=0.5, above=1pt] {-} (a);
        \draw[-stealth, thick] (c) -- node[pos=0.5, above=1pt] {+} (a);        
    \end{tikzpicture}
    \caption{$\sctrbrempty, \sctrbriempty, \sctrbsempty$ violate the monotonicity principle w.r.t. QE, DFQuAD, SD-DFQuAD, EB, and EBT semantics $\fs$.}
    \label{fig:proof-monotonicity-negative}
\end{figure}
As $\sctrbgmempty$ takes the maximum of a set of single-argument gradient-based contributions, it increases weakly monotonically when expanding the set contributor, and thus satisfies monotonicity. 
\begin{proposition}\label{prop:monotonicity-positive}
    $\sctrbgmempty$ satisfies the monotonicity principle w.r.t. QE, DFQuAD, SD-DFQuAD, EB, and EBT semantics $\fs$.
\end{proposition}
\begin{proof}
    Analogously to the proof of Proposition~\ref{prop:consistency-positive}, we observe that $\sctrbgmempty$ is the maximal gradient of a single argument in a set of contributors, and we again simplify notation by denoting $\sctrbgm{\argX}{\arga}$ by $\max_{\argx \in \argX} g(\argx)$.
    Clearly if $\argX \subseteq \argY$, it holds that $\max_{\argx \in \argX} g(\argx) \leq \max_{\argy \in \argY} g(\argy)$, and thus monotonicity (Principle~\ref{principle:monotonicity}) is satisfied.
\end{proof}
Clearly, applying other pooling functions than $\max$ for computing gradient-based set contributions may have an effect on monotonicity.
For example, when using $\min$ instead, monotonicity would be violated, but an analogous principle requiring a monotonic \emph{decrease} of the set contribution function would be satisfied.
Applying the maximal \emph{absolute} as a pooling function would satisfy monotonicity.

\subsection{Symmetry}
\label{subsec:symmetry}
Even given semantics for which equal positive and negative influence ``balances out'' (in the sense of~\cite{PB24}, Definition 5), the contributions of $\sctrbriempty$, $\sctrbsempty$ do not achieve symmetry, and $\sctrbgmempty$ only achieves symmetry under QE semantics. 
Intuitively, the problem is that even given balanced semantics, the union of the direct attackers and direct supporters is evaluated as a single set, which may not always preserve the balance between the two sides.
\begin{proposition}\label{prop:symmetry-positive-remove}
    $\sctrbrempty$ satisfies the symmetry principle w.r.t. QE, DFQuAD, SD-DFQuAD, EB, and EBT semantics $\fs$.
\end{proposition}
\begin{proof}
    Let $\argX = {\cal R}^-(\arga) \cup {\cal R}^+(\arga)$. 
    Assume that the multisets $\{\fs(\argx) | \argx \in {\cal R}^-(\arga)\}$ and $\{\fs(\argx) | \argx \in {\cal R}^+(\arga)\}$ are equal.
    Then, the aggregated input at $\arga$ is $0$.
    This is because, for Sum, equality of the multisets implies equality of the sums of attacker and supporter strengths; for Product, equality of the multisets implies equality of the two products; and for Top, equality of the multisets implies equality of the two maxima.
    Hence the aggregated input at $\arga$ is $0$, and therefore $\fs(\arga) = \is(\arga)$ for the QE, DFQuAD, SD-DFQuAD, EB, and EBT semantics.
    Moreover, in the restricted graph $\graph\downarrow_{\Args \setminus \argX}$, all direct attackers and supporters of $\arga$ have been removed, so its final strength is again just its initial strength: $\fs_{\graph\downarrow_{\Args \setminus \argX}}(\arga) = \is(\arga)$, since all considered semantics satisfy bi-variate independence and bi-variate directionality (Lemma~\ref{lemma:semantics-bi-variate-principles}).
    Therefore,
    $\sctrbr{\argX}{\arga} = \fs(\arga) - \fs_{\graph\downarrow_{\Args \setminus \argX}}(\arga) = \is(\arga) - \is(\arga) = 0$.
    Hence, $\sctrbrempty$ satisfies the symmetry principle.
\end{proof}

Unlike in the case of $\sctrbrempty$, the balance between direct attackers and supporters need not be preserved for the intrinsic-removal variant $\sctrbriempty$. 
Although ${\cal R}^-(\arga)$ and ${\cal R}^+(\arga)$ may initially be balanced, removing the attackers and supporters of arguments in either ${\cal R}^-(\arga)$ or ${\cal R}^+(\arga)$ can disturb that balance.
\begin{proposition}\label{prop:symmetry-negative}
    $\sctrbriempty$ violates the symmetry principle w.r.t. QE, DFQuAD, SD-DFQuAD, EB, and EBT semantics $\fs$.
\end{proposition}
\begin{proof}
    Consider the QBAG displayed in \autoref{fig:proof-symmetry}.
    According to the definition of the \emph{symmetry} principle, the contribution of the union of all direct supporters and attackers $\{\argb, \argc\}$ to the topic should equal to $0$.
    From the counter-examples we observe that this is not the case.
\end{proof}
\begin{figure}[h!]
    \centering
    \begin{tikzpicture}[node distance=1.5cm]
        \node[unanode] (a) at (0, 0) {\argnode{\arga}{\is(\arga)}{\fs(\arga)}};
        \node[unanode] (b) at (-2, 0) {\argnode{\argb}{\is(\argb)}{n}};
        \node[unanode] (c) at (2, 0) {\argnode{\argc}{\is(\argc)}{n}};
        \node[unanode] (d) at (4, 0) {\argnode{\argd}{\is(\argd)}{\fs(\argd)}};

        \draw[-stealth, thick] (b) -- node[pos=0.5, above=1pt] {-} (a);
        \draw[-stealth, thick] (c) -- node[pos=0.5, above=1pt] {+} (a);
        \draw[-stealth, thick] (d) -- node[pos=0.5, above=1pt] {-} (c);
    
    \end{tikzpicture}
    \caption{$\sctrbriempty$ violates the symmetry principle w.r.t. QE, DFQuAD, SD-DFQuAD, EB, and EBT semantics $\fs$. Arguments $\argb$ and $\argc$ have the same final strength $n$.}
    \label{fig:proof-symmetry}
\end{figure}
As an example, take DFQuAD semantics, which yields $\sctrbri{\{b, c\}}{a} = 0.25$.

\begin{proposition}\label{prop:symmetry-negative-shapley}
    $\sctrbsempty$ violates the symmetry principle w.r.t. QE, DFQuAD, SD-DFQuAD, EB, and EBT semantics $\fs$.
\end{proposition}
\begin{proof}
    Consider the QBAG displayed in \autoref{fig:proof-symmetry-shapley}.
    According to the definition of the \emph{symmetry} principle, the union of all direct supporters and attackers $\{\argb, \argc, \argd, \arge\}$ to the topic should equal to $0$.
    From the counter-examples we observe that this is not the case.
\end{proof}
\begin{figure}[h!]
    \centering
    \begin{tikzpicture}[node distance=1.5cm]
        \node[unanode] (a) at (0, 0) {\argnode{\arga}{\is(\arga)}{\fs(\arga)}};
        \node[unanode] (b) at (-2, 1) {\argnode{\argb}{0.5}{n}};
        \node[unanode] (c) at (-2, -1) {\argnode{\argc}{0.5}{0.5}};
        \node[unanode] (e) at (2, 1) {\argnode{\arge}{0.5}{n}};
        \node[unanode] (d) at (2, -1) {\argnode{\argd}{0.5}{0.5}};
        \node[unanode] (f) at (-4, 1) {\argnode{\argf}{0.5}{0.5}};
        \node[unanode] (g) at (4, 1) {\argnode{\argg}{0.5}{0.5}};

        \draw[-stealth, thick] (b) -- node[pos=0.5, above=1pt] {-} (a);
        \draw[-stealth, thick] (c) -- node[pos=0.5, above=1pt] {-} (a);
        \draw[-stealth, thick] (e) -- node[pos=0.5, above=1pt] {+} (a);
        \draw[-stealth, thick] (d) -- node[pos=0.5, above=1pt] {+} (a);
        \draw[-stealth, thick] (f) -- node[pos=0.5, above=1pt] {-} (b);
        \draw[-stealth, thick] (g) -- node[pos=0.5, above=1pt] {-} (e);
    
    \end{tikzpicture}
    \caption{$\sctrbsempty$ violates the symmetry principle w.r.t. QE, DFQuAD, SD-DFQuAD, EB, and EBT semantics $\fs$. Arguments $\argb$ and $\arge$ have the same final strength $n$.}
    \label{fig:proof-symmetry-shapley}
\end{figure}

Since $\sctrbgmempty$ reduces each side to the argument with the largest positive gradient, it can ignore the balanced attacker/supporter structure that symmetry is intended to capture. Whether this balance is preserved in the resulting final strengths, however, depends on the semantics.
\begin{proposition}\label{prop:symmetry-positive-gradient}
    $\sctrbgmempty$ satisfies the symmetry principle w.r.t. QE semantics $\fs$.
\end{proposition}
\begin{proof}
Assume that the multisets of the strengths of the direct attackers and supporters of $\arga$ are equal, i.e. $\{\fs(\argx) | \argx \in {\cal R}^-(\arga)\}$ and $\{\fs(\argx) | \argx \in {\cal R}^+(\arga)\}$.
Then the total attacking strength equals the total supporting strength.
Under QE semantics, aggregation is by sum, so the aggregated input at $\arga$ is $0$.
QE uses the influence function 2-Max(1).
The derivative of 2-Max(1) at any point where the aggregation function returns $0$ is $0$, beucase 2-Max(1) uses $h(x)$ with $p=2$.
Hence, by the chain rule, the partial derivative of the strength of $\arga$ with respect to 
the initial strength of each direct attacker and supporter is  $0$.
Taking the maximum over all attackers and supporters yields
$\sctrbgm{{\cal R}^-(\arga) \cup {\cal R}^+(\arga)}{a} = 0$.
Thus, the set gradient contribution function satisfies the symmetry principle w.r.t. QE semantics.
\end{proof}

Under QE, the balancing effect is preserved by the gradient-based contribution measure $\sctrbgmempty$.
By contrast, this behavior is not preserved under DFQuAD, SD-DFQuAD, EB, and EBT semantics.
\begin{proposition}\label{prop:symmetry-gradient-negative}
    $\sctrbgmempty$ violates the symmetry principle w.r.t. DFQuAD, SD-DFQuAD, EB, and EBT semantics $\fs$.
\end{proposition}
\begin{proof}
    Consider the QBAG displayed in \autoref{fig:symmetry-gradient-QE-negative}.
    According to the definition of the \emph{symmetry} principle, the contribution of the union of all direct supporters and attackers $\{\argb, \argc\}$ to the topic should equal to $0$.
    From the counter-example we observe that this is not the case.
\end{proof}

\begin{figure}[h!]
    \centering
    \begin{tikzpicture}[node distance=1.5cm]
        \node[unanode] (a) at (0, 0) {\argnode{\arga}{0.5}{0.5}};
        \node[unanode] (b) at (-2, 0) {\argnode{\argb}{0.5}{0.5}};
        \node[unanode] (c) at (2, 0) {\argnode{\argc}{0.5}{0.5}};
        
        \draw[-stealth, thick] (b) -- node[pos=0.5, above=1pt] {-} (a);
        \draw[-stealth, thick] (c) -- node[pos=0.5, above=1pt] {+} (a);
    \end{tikzpicture}
    \caption{$\sctrbgmempty$ violates the symmetry principle w.r.t. DFQuAD, SD-DFQuAD, EB, and EBT semantics $\fs$.}
    \label{fig:symmetry-gradient-QE-negative}
\end{figure}

\section{Illustrative Application}
\label{sec:application}
In this section, we sketch an application scenario illustrating how set contribution functions can be applied. 
The scenario, together with Example~\ref{ex:intro}, helps us highlight differences between single-argument and set contribution functions, as well as the relevance of some of the differences between set contribution functions that principles give rise to. 
The application is based on~\cite{sukpanichnant2025peerargargumentativepeerreview}, which incorporates Natural Language Processing (NLP) techniques into QBAG-based classification. 
We apply our approach to an example QBAG reflecting one that has been reported in the above work. 
However, the approach is applicable to generic scenarios and~\cite{sukpanichnant2025peerargargumentativepeerreview} links to open-source code that can be used for QBAG instantiation from data (in contrast, we manually specify the initial QBAG for the sake of simplicity).
Although set-based contribution functions may be more interesting in scenarios with many inter-argument interactions, this example application illustrates, again together with Example~\ref{ex:intro}, core differences between single-argument and set-based contribution functions.

We apply set contribution functions to a QBAG representing reviews and gradual acceptance/rejection recommendations for scientific papers. 
Note that the use-case is purely illustrative and we do not recommend an immediate application to the scenario at hand, which requires more comprehensive evaluations and a careful consideration of social/human aspects. 
Consider the QBAGs displayed in Subfigure~\ref{fig:QBAF-example-application-step2}, depicting a two-layered QBAG. The topic argument in the ``bottom'' layer represents the review recommendation, which can be discretized, i.e. (roughly), mapped to recommendations from \textit{strong reject} to \textit{strong accept}. 
Its initial strength is $0.5$ (which would reflect a \textit{borderline} evaluation if interpreted as a final strength). 
The layer above represents \textit{aspects} that the review may cover, which are provided as a generic meta-model for this use-case: \textit{appropriateness}, \textit{clarity}, \textit{novelty}, \textit{empirical} and \textit{theoretical soundness}, \textit{meaningful comparison}, \textit{substance}, and \textit{impact}, abbreviated as \texttt{APR}, \texttt{CLA}, \texttt{NOV}, \texttt{EMP}, \texttt{CMP}, \texttt{SUB}, and \texttt{IMP}, respectively. 
These arguments are also initiated with a strength of $0.5$ each. 
Finally, a third layer, depicted in Subfigure~\ref{fig:QBAF-example-application-step1}, presents three text arguments representing sentences and their sentiment strengths, which, including final strengths and alongside the support and attack relationships to the \textit{aspect} layer, are mined---in the aforementioned paper---using NLP approaches.
Each aspect argument attacks or supports the decision argument, depending on whether its final strength is above or below $0.5$ (respectively). 
This means we first compute the final strengths of the aspects arguments before determining their relationships with the decision argument.
Also, before the computation of the decision argument's final strength, the final strengths of the aspect arguments are normalized to a strength value in $(0 , 1]$ based on the distance to the initial value of $0.5$. 
Strengths are updated using DF-QuAD semantics.

As the computation process is broken down into two steps that, when combined, do not correspond to the application of a gradual semantics to a QBAG, we focus on the contribution quantification within the second step, i.e., determining the contribution of aspect arguments to the decision (Subfigure~\ref{fig:QBAF-example-application-step2}).
For the combined approach, we could practically still apply the contribution functions, but the application would break the frame of our analytical results with respect to principle satisfaction.
\autoref{table:QBAF-example-application-results} displays the results of computing the different set contributions of the set \texttt{\{~NOV,~IMP~\}} on the topic, as well as the single-argument contributions of \texttt{NOV} and \texttt{IMP} (separated) and their sum. 
We observe that the contributions distributed to the players \texttt{\{~NOV,~IMP~\}}, \texttt{CMP}, and \texttt{APR} by the Shapley value approximately account for the difference between our topic argument's final and initial strengths, as expected.
\begin{table}[ht]
    \centering
    \caption{Contributions to the topic argument $\argD$ from $\sctrbrempty, \sctrbsempty$, and $\sctrbgmempty$ (rounded to three decimals).}
    \begin{tabular}{lccc}
        \hline
        Contributors & $\sctrbrempty$ & $\sctrbsempty$ & $\sctrbgmempty$ \\
        \hline
        $\{ \mathrm{NOV}, \mathrm{IMP} \}$ & 0.045  & 0.048  & 0.200  \\
        $\mathrm{NOV}$                     & 0.120  & 0.210  & 0.200  \\
        $\mathrm{IMP}$                     & -0.075 & -0.163 & -0.150 \\
        $\mathrm{CMP}$                     & -0.175 & -0.263 & -0.250 \\
        $\mathrm{APR}$                     & 0.120  & 0.210  & 0.200  \\
        $\mathrm{CMP} + \mathrm{APR} + \{\mathrm{NOV},\mathrm{IMP}\}$ & -0.010 & -0.005 & 0.150 \\
        $\{\mathrm{CMP}, \mathrm{APR}, \mathrm{NOV},\mathrm{IMP}\}$ & -0.005 & -0.005 & 0.200 \\
        \hline
    \end{tabular}
    \label{table:QBAF-example-application-results}
\end{table}

From a use-case perspective, the idea is to focus on novelty and impact, as their assessment may be rather subjective.
We drop the \textit{intrinsic removal}-based set contribution function here, as its contributions clearly coincide with the ones of the removal-based variant. 
Note that an application example where the difference between intrinsic removal and removal-based contributions matter is presented for the single-argument case in \cite{DBLP:journals/ijar/KampikPYCT24}.
As we can see from the results of the set contributor containing all arguments, set contributions can differ from the single argument contributions and their sum-based aggregations even when no inter-argument interactions are present (see Example~\ref{ex:intro} for an additional illustration of set-based contributions including inter-argument interactions).
However, the gradient set contribution ($\sctrbgmempty$) of \texttt{\{~NOV,~IMP~\}} on the topic equals the single-argument contribution of $0.2$ (as expected).

\begin{figure}[!ht]
    \centering
    \subfloat[Initial QBAG before aspect-decision relationships are determined. No relations were extracted for the \texttt{CLA}, \texttt{EMP}, and \texttt{SUB} \textit{aspects}.]{
        \centering
            \begin{tikzpicture}[scale=0.7]
                \node[unanode] (t1) at (-2, 3) {\argnode{t1}{0.6}{0.6}};
                \node[unanode] (t2) at (1, 3) {\argnode{t2}{0.7}{0.7}};
                \node[unanode] (t3) at (3, 3) {\argnode{t3}{0.5}{0.5}};
                \node[unanode] (NOV) at (-3, 0) {\argnode{NOV}{0.5}{0.8}};
                \node[unanode] (CMP) at (1, 0) {\argnode{CMP}{0.5}{0.15}};
                \node[unanode] (APR) at (-1, 0) {\argnode{APR}{0.5}{0.8}};
                \node[unanode] (IMP) at (3, 0) {\argnode{IMP}{0.5}{0.25}};
                \node[unanode] (SUB) at (2, -2) {\argnode{SUB}{0}{0}};
                \node[unanode] (EMP) at (0, -2) {\argnode{EMP}{0}{0}};
                \node[unanode] (CLA) at (-2, -2) {\argnode{CLA}{0}{0}};
    
                \draw[-stealth, thick] (t1) -- node[pos=0.5, left=1pt] {+} (NOV);
                \draw[-stealth, thick] (t1) -- node[pos=0.5, right=1pt] {+} (APR);
                \draw[-stealth, thick] (t2) -- node[pos=0.5, left=1pt] {-} (CMP);
                \draw[-stealth, thick] (t3) -- node[pos=0.5, left=1pt] {-} (IMP);
            \end{tikzpicture}
        \label{fig:QBAF-example-application-step1}
    }
    \hspace{5pt}
    \centering
    \subfloat[QBAG after aspect normalization and relationship determination.]{
        \centering
            \begin{tikzpicture}[scale=0.7]
                \node[unanode] (NOV) at (-3, 0) {\argnode{NOV}{0.6}{0.6}};
                \node[unanode] (CMP) at (1, 0) {\argnode{CMP}{0.7}{0.7}};
                \node[unanode] (APR) at (-1, 0) {\argnode{APR}{0.6}{0.6}};
                \node[unanode] (IMP) at (3, 0) {\argnode{IMP}{0.5}{0.5}};
                \node[unanode] (D) at (0, -3) {\argnode{D}{0.5}{0.495}};
    
                \draw[-stealth, thick] (NOV) -- node[pos=0.5, left=1pt] {+} (D);
                \draw[-stealth, thick] (CMP) -- node[pos=0.5, left=1pt] {-} (D);
                \draw[-stealth, thick] (APR) -- node[pos=0.5, left=1pt] {+} (D);
                \draw[-stealth, thick] (IMP) -- node[pos=0.5, left=1pt] {-} (D);
            \end{tikzpicture}
        \label{fig:QBAF-example-application-step2}
    }
    \caption{QBAG construction in a recommendation systems example.}
\end{figure}

Overall, these results indicate that there is a meaningful difference between set contribution and single-argument contribution functions.
Between set contribution functions, we see different purposes and effects:
\begin{itemize}
    \item The removal-based contribution of $0.045$ describes what would happen if we were to ignore novelty and impact assessment, reflecting the principle of strong counterfactuality: we can see that the absence of these aspects would decrease the final strengths of the topic argument accordingly. 
    In our use-case this means that the paper, whose score corresponds to a somewhat \textit{weak borderline} recommendation, would be further weakened in absence of these rather subjective aspects, which may help a meta-reviewer interpret the review as being more on the ``reject side'' of the scale.
    \item The Shapley value-based contribution of $0.048$ indicates how much novelty and impact assessments contribute, as a ``group'' to the overall score, relative to the contributions of the other aspects, which are each grouped separately. 
    The soundness of the contribution is reflected by the weak quantitative contribution existence principle. 
    However, if we want to compare the contribution relative to another set of several aspects, we would need to make use of the partition contribution function ($\pctrbsempty$). 
    In our example, Shapley-based contribution and removal-based contribution support the same narrative: the two somewhat opinionated aspects \textit{novelty} and \textit{impact} positively affect the score.
    \item The gradient-based contribution function (yielding a contribution of $0.2$) quantifies the effectiveness of applying a change to the aspect among \texttt{\{~NOV,~IMP~\}} that has a locally greater effect.
    For this, we would expect behavior that amounts to monotonicity. For the opposite direction (how to best affect a lower score), we would need to use \texttt{min} as our pooling function of the gradient.
    \item Finally, in this example, the three set contribution functions induce somewhat similar rankings over most contributors/sets.
    However, these rankings do not coincide in all cases.
    Consider the set contributor $\{ \mathrm{NOV}, \mathrm{IMP} \}$, which is ranked third by the removal- and Shapley value-based approaches, while the gradient-max set contribution ranks it first.
    The mismatch reflects that $\sctrbgmempty$ ranks sets by the \emph{locally strongest effect} inside the set, whereas removal- and Shapley value-based contributions measure the sets overall influence under counterfactual/game-theoretic conditions.
    Consequently, adding further arguments to a set does not change the impact measured by $\sctrbgmempty$ when these arguments' local effects do not exceed the current maximum; this can lead to different rankings compared to the other set contribution functions.
\end{itemize}

The corresponding implementation of this example and additional examples where the induced rankings differ, can be found at \url{https://github.com/TimKam/Quantitative-Bipolar-Argumentation/tree/main/analysis}.

\section{Discussion and Conclusion}
\label{sec:discussion}
Our work can be considered a contribution to \emph{argumentative explainability} that encompasses both the application of formal argumentation in order to facilitate explainability and the theoretical study of explanations in formal argumentation~\cite{Cyras.et.al:2021-IJCAI,vassiliades_bassiliades_patkos_2021}.
Our results contribute primarily to the latter line of research, although our application examples provide some indication of potential usefulness to the former.
Over the past years, contribution functions for argumentation-based inference---sometimes instead called \emph{attribution explanations}~\cite{DBLP:conf/ecai/0007PT23} (where the focus is on a gradient-based single-argument contribution function under the DF-QuAD semantics for acyclic QBAGs) or \emph{impact measures}~\cite{DBLP:conf/ifaamas/AnaissyDVY25}---have been the subject of substantial studies.
Several recent works study \emph{principles} for such functions.
Perhaps most notably, \cite{DBLP:conf/ifaamas/AnaissyDVY25}~does so for gradual argumentation given initially unweighted graphs,
whereas \cite{DBLP:journals/ijar/KampikPYCT24}~focuses on acyclic QBAGs.
In~\cite{Delobelle:Villata:2019}, the authors introduce set contribution functions based on notions of removal and intrinsic removal (a variant stronger than ours, since it removes arguments based on internal conflicts as well) and analyzes them in relation to the \emph{h-categorizer} and \emph{counting} semantics.
However, to the best of our knowledge, our study is the first comprehensive work focusing on contributions of sets of arguments to a topic argument in quantitative bipolar argumentation across several set contribution function variants.
Given our results, we argue that studying the more general set contribution functions in addition to the already well-studied single argument contribution functions is interesting:
not only are set contribution function more general than single-argument contribution functions; also, there are nontrivial differences between the single-argument and set contributor case that warrant the study of both.

We can see this when providing intuitive \emph{generalizations} of single-argument contribution function principles for the set contribution function case.
Notably, the \emph{contribution existence} principle is, when satisfied for the set contribution cases, often violated by single-argument specializations of the corresponding contribution functions.
Also, we have introduced new principles that focus on interactions between sets and thus only apply to set contribution functions; these principles can meaningfully distinguish between most set contribution functions.
Beyond that, our studies have unearthed the potential for further generalization, namely in the form of \emph{partition contribution functions} that determine set contributions relative to a partition of the set of all arguments except the topic.

Note that we did not cover some of the single-argument contribution functions that have been proposed in previous works;
examples are:
\begin{enumerate*}[label=\roman*)]
  \item \emph{language independence}~\cite{DBLP:conf/ifaamas/AnaissyDVY25}, intuitively: changing the symbols used for arguments while retaining their relationships and initial strengths does not affect their final strengths;
  \item \emph{faithfulness}~\cite{DBLP:journals/ijar/KampikPYCT24}, roughly and in different variants: the effect of (small) changes to a contributor's initial strength have an effect on the topic's final strength that is somewhat consistent with the computed contribution score.
\end{enumerate*}
This is motivated by the need to keep the scope manageable while maintaining somewhat interesting results.
Clearly, language independence is trivially satisfied by all of our contribution functions, and based on what we know from the results in~\cite{DBLP:journals/ijar/KampikPYCT24}, faithfulness is primarily interesting for the gradient-based contribution function case.
Still, we hope that the study of principles in the literature that we do not cover can reveal further interesting results in the future.
Also, we admit that some of the results we have obtained, in particular with respect to directionality, are rather mundane; let us claim that some mundane results are difficult to avoid when carrying out comprehensive principle-based analyses.

More broadly, our work can be considered a contribution to the \emph{principle-based} study of formal argumentation, somewhat analogous to for example~\cite{van2017principle,Baroni:Rago:Toni:2019}, which focus on principles for abstract and gradual semantics, respectively.
Also, like most works focusing on formal aspects of argumentative explainability, we examine the effects of how changes to argumentation graphs affect the inferences we draw from them; accordingly, our results contribute to the study of \emph{argumentation dynamics} (cf.~\cite{doutre-argument} for a survey).

Future work may extend our results in different directions.
For example, one may provide a principle-based analysis that also considers the case of cyclic QBAGs to broaden the applicability beyond acyclic graphs, or one could define and study additional partition contribution functions that determine relative contributions of disjoint sets of arguments and are based on power indices other than the Shapley value (for which we have defined such function).

\subsubsection*{Acknowledgments}
\label{sec:acks}
We thank the anonymous reviewers for constructive remarks that have helped us improve this paper.

This work was partially supported by the Wallenberg AI, Autonomous Systems and Software Program (WASP) funded by the Knut and Alice Wallenberg Foundation.

\small
\bibliographystyle{elsarticle-num} 
\bibliography{references}

\appendix
\renewcommand{\thesubsection}{\Alph{subsection}}
\section{Appendix: Counterexamples for (Quantitative) Counterfactuality}
\label{sec:appendix-counterexamples}
%
%
\begin{counterexample}[(Quantitative) Counterfactuality; $\sctrbriempty$ w.r.t. QE, DFQuAD, and SD-DFQuAD]\label{ce:iremove-counterfactuality}
    Consider the QBAG displayed in \autoref{fig:proof-counterfactuality}, which we denote by $\graph = (\Args, \is, \Att, \Supp)$. Given QE, DFQuAD, and SD-DFQuAD semantics $\fs$, we observe that $\fs_{\graph}(a) < \fs_{\graph \downarrow_{\Args \setminus \{b\}}}(a)$. 
    According to the definitions of counterfactuality and quantitative counterfactuality, under these conditions $\sctrbri{\{\argb\}}{\arga} > 0$ must hold for the principles to be satisfied. 
    However, we have $\sctrbri{\{\argb\}}{\arga} = 0$, demonstrating the violation of the counterfactuality and quantitative counterfactuality principles.
\end{counterexample}
\begin{figure}[!ht]
    \centering
    \begin{tikzpicture}[node distance=1.5cm]
        \node[unanode] (a) at (2, 0) {\argnode{a}{1.0}{< 1.0}};
        \node[unanode] (b) at (0, 0) {\argnode{b}{0.0}{> 0.0}};
        \node[unanode] (c) at (-2, 0) {\argnode{c}{1.0}{1.0}};

        \draw[-stealth, thick] (c) -- node[pos=0.5, above=1pt] {+} (b);
        \draw[-stealth, thick] (b) -- node[pos=0.5, above=1pt] {-} (a);
    
    \end{tikzpicture}
    \caption{$\sctrbriempty$ violates (quantitative) counterfactuality w.r.t. QE, DFQuAD, and SD-DFQuAD semantics $\fs$.}
    \label{fig:proof-counterfactuality}
\end{figure}

\begin{counterexample}[(Quantitative) Counterfactuality; $\sctrbriempty$ w.r.t. EB]\label{ce:iremove-counterfactuality-EB}
    Consider the QBAG displayed in \autoref{fig:proof-iremove-counterfactuality-EB}, which we denote by $\graph = (\Args, \is, \Att, \Supp)$. Given EB semantics, we observe that $\sctrbri{\{e\}}{\arga} \approx 3.5431 \times 10^{-6}$ (i.e. $\sctrbri{\{e\}}{\arga} > 0$). 
    Hence, according to the definitions of counterfactuality and quantitative counterfactuality, we would expect $\fs_{\graph}(a) > \fs_{\graph \downarrow_{\Args \setminus \{e\}}}(a)$. 
    However, this does not hold, thus demonstrating the violation of the principles.
\end{counterexample}
\begin{figure}[!ht]
    \centering
    \begin{tikzpicture}[node distance=1.5cm]
        \node[unanode] (f) at (0, 2) {\argnode{f}{1.0}{1.0}};
        \node[unanode] (e) at (2, 2) {\argnode{e}{0.02}{0.005}};
        \node[unanode] (b) at (4, 0) {\argnode{b}{0.1}{0.104}};
        \node[unanode] (d) at (2, 0) {\argnode{d}{0.51}{0.519}};
        \node[unanode] (c) at (0, 0) {\argnode{c}{0.1}{0.104}};
        \node[unanode] (g) at (-2, 0) {\argnode{g}{0.27}{0.270}};
        \node[unanode] (a) at (0, -2) {\argnode{a}{0.5}{0.507}};
        
        \draw[-stealth, thick] (f) -- node[pos=0.5, above=1pt] {+} (e);
        
        \draw[-stealth, thick] (e) -- node[pos=0.5, above=1pt] {+} (b);
        \draw[-stealth, thick] (e) -- node[pos=0.5, right=1pt] {+} (d);
        \draw[-stealth, thick] (e) -- node[pos=0.5, above=1pt] {+} (c);
        
        \draw[-stealth, thick] (b) -- node[pos=0.5, below=1pt] {-} (a);
        \draw[-stealth, thick] (d) -- node[pos=0.5, above=1pt] {+} (a);
        \draw[-stealth, thick] (c) -- node[pos=0.5, right=1pt] {-} (a);
        \draw[-stealth, thick] (g) -- node[pos=0.5, above=1pt] {-} (a);
    
    \end{tikzpicture}
    \caption{$\sctrbriempty$ violates (quantitative)  counterfactuality w.r.t. EB semantics $\fs$.}
    \label{fig:proof-iremove-counterfactuality-EB}
\end{figure}

\begin{counterexample}[(Quantitative) Counterfactuality; $\sctrbriempty$ w.r.t. EBT]\label{ce:iremove-counterfactuality-EBT}
    Consider the QBAG displayed in \autoref{fig:proof-iremove-counterfactuality-EBT}, which we denote by $\graph = (\Args, \is, \Att, \Supp)$. Given EBT semantics, we observe that $\sctrbri{\{b\}}{\arga} =  0$. 
    According to the definition of the counterfactuality and quantitative counterfactuality principles, we would expect $\fs_{\graph}(a) = \fs_{\graph \downarrow_{\Args \setminus \{b\}}}(a)$. However, this does not hold, thus illustrating the violation of the counterfactuality and quantitative counterfactuality principles.
\end{counterexample}
\begin{figure}[!ht]
    \centering
    \begin{tikzpicture}[node distance=1.5cm]
        \node[unanode] (a) at (0, 0) {\argnode{a}{0.7}{0.673}};
        \node[unanode] (b) at (2, 0) {\argnode{b}{0.1}{0.221}};
        \node[unanode] (c) at (4, 0) {\argnode{c}{1.0}{1.0}};
        \node[unanode] (d) at (-2, 0) {\argnode{d}{0.1}{0.1}};

        \draw[-stealth, thick] (c) -- node[pos=0.5, above=1pt] {+} (b);
        \draw[-stealth, thick] (b) -- node[pos=0.5, above=1pt] {-} (a);
        \draw[-stealth, thick] (d) -- node[pos=0.5, above=1pt] {-} (a);
    
    \end{tikzpicture}
    \caption{$\sctrbriempty$ violates (quantitative) counterfactuality w.r.t. EBT semantics $\fs$.}
    \label{fig:proof-iremove-counterfactuality-EBT}
\end{figure}

%
%
\begin{counterexample}[(Quantitative) Counterfactuality; $\sctrbsempty$ w.r.t. QE]\label{ce:shapley-counterfactuality-QE}
Consider the QBAG displayed in \autoref{fig:proof-shapley-counterfactuality-QE}, which we denote by $\graph = (\Args, \is, \Att, \Supp)$.
    Given QE semantics, we observe that $\sctrbs{\{e\}}{\arga} \approx  4.9326 \times 10^{-5}$. 
    According to the definitions of counterfactuality and quantitative counterfactuality we must have $\fs_{\graph}(a) > \fs_{\graph \downarrow_{\Args \setminus \{e\}}}(a)$, which does not hold. 
    This illustrates the violation of the counterfactuality and quantitative counterfactuality principles.
\end{counterexample}
\begin{figure}[!ht]
    \centering
    \begin{tikzpicture}[node distance=1.5cm]
        \node[unanode] (f) at (-4, 0) {\argnode{f}{1.0}{1.0}};
        \node[unanode] (e) at (-2, 0) {\argnode{e}{0.495}{0.7475}};
        \node[unanode] (d) at (0, 0) {\argnode{d}{0.15}{0.454693}};
        \node[unanode] (b) at (0, 2) {\argnode{b}{0.15}{0.454693}};
        \node[unanode] (c) at (0, -2) {\argnode{c}{0.15}{0.454693}};
        \node[unanode] (a) at (2, 0) {\argnode{a}{0.1}{0.082867}};

        \draw[-stealth, thick] (f) -- node[pos=0.5, above=1pt] {+} (e);
        \draw[-stealth, thick] (e) -- node[pos=0.5, left=1pt] {+} (b);
        \draw[-stealth, thick] (e) -- node[pos=0.5, above=1pt] {+} (d);
        \draw[-stealth, thick] (e) -- node[pos=0.5, right=1pt] {+} (c);
        \draw[-stealth, thick] (b) -- node[pos=0.5, right=1pt] {-} (a);
        \draw[-stealth, thick] (d) -- node[pos=0.5, above=1pt] {+} (a);
        \draw[-stealth, thick] (c) -- node[pos=0.5, left=1pt] {-} (a);
    
    \end{tikzpicture}
    \caption{$\sctrbsempty$ violates (quantitative) counterfactuality w.r.t. QE semantics $\fs$.}
    \label{fig:proof-shapley-counterfactuality-QE}
\end{figure}

\begin{counterexample}[(Quantitative) Counterfactuality; $\sctrbsempty$ w.r.t. DFQuAD]\label{ce:shapley-counterfactuality-DFQuAD}
    Consider the QBAG displayed in \autoref{fig:proof-shapley-counterfactuality-DFQuAD}, which we denote by $\graph = (\Args, \is, \Att, \Supp)$.
    Given DFQuAD semantics, we observe that $\sctrbs{\{e\}}{\arga} \approx  0.0027$. 
    According to the definitions of counterfactuality and quantitative counterfactuality we must have $\fs_{\graph}(a) > \fs_{\graph \downarrow_{\Args \setminus \{e\}}}(a)$, which does not hold. 
    This proves the violation of the counterfactuality and quantitative counterfactuality principles.
\end{counterexample}
\begin{figure}[!ht]
    \centering
    \begin{tikzpicture}[node distance=1.5cm]
        \node[unanode] (f) at (-4, 0) {\argnode{f}{1.0}{1.0}};
        \node[unanode] (e) at (-2, 0) {\argnode{e}{0.495}{1.0}};
        \node[unanode] (d) at (0, 0) {\argnode{d}{0.3}{1.0}};
        \node[unanode] (b) at (0, 2) {\argnode{b}{0.15}{1.0}};
        \node[unanode] (c) at (0, -2) {\argnode{c}{0.17}{1.0}};
        \node[unanode] (a) at (2, 0) {\argnode{a}{0.1}{0.1}};

        \draw[-stealth, thick] (f) -- node[pos=0.5, above=1pt] {+} (e);
        \draw[-stealth, thick] (e) -- node[pos=0.5, left=1pt] {+} (b);
        \draw[-stealth, thick] (e) -- node[pos=0.5, above=1pt] {+} (d);
        \draw[-stealth, thick] (e) -- node[pos=0.5, right=1pt] {+} (c);
        \draw[-stealth, thick] (b) -- node[pos=0.5, right=1pt] {-} (a);
        \draw[-stealth, thick] (d) -- node[pos=0.5, above=1pt] {+} (a);
        \draw[-stealth, thick] (c) -- node[pos=0.5, left=1pt] {-} (a);
    \end{tikzpicture}
    \caption{$\sctrbsempty$ violates (quantitative) counterfactuality w.r.t. DFQuAD semantics $\fs$.}
    \label{fig:proof-shapley-counterfactuality-DFQuAD}
\end{figure}
\begin{counterexample}[(Quantitative) Counterfactuality; $\sctrbsempty$ w.r.t. SD-DFQuAD]\label{ce:shapley-counterfactuality-SD-DFQuAD}
    Consider the QBAG displayed in \autoref{fig:proof-shapley-counterfactuality-SD-DFQuAD}, which we denote by $\graph = (\Args, \is, \Att, \Supp)$.
    Given SD-DFQuAD semantics, we observe that $\sctrbs{\{e\}}{\arga} \approx  0.0022$. 
    According to the definitions of counterfactuality and quantitative counterfactuality we must have $\fs_{\graph}(a) > \fs_{\graph \downarrow_{\Args \setminus \{e\}}}(a)$, which does not hold. 
    This illustrates the violation of the counterfactuality and quantitative counterfactuality principles.
\end{counterexample}
\begin{figure}[!ht]
    \centering
    \begin{tikzpicture}[node distance=1.5cm]
        \node[unanode] (f) at (-4, 0) {\argnode{f}{1.0}{1.0}};
        \node[unanode] (e) at (-2, 0) {\argnode{e}{0.495}{0.7475}};
        \node[unanode] (d) at (0, 0) {\argnode{d}{0.2}{0.542203}};
        \node[unanode] (b) at (0, 2) {\argnode{b}{0.15}{0.513591}};
        \node[unanode] (c) at (0, -2) {\argnode{c}{0.15}{0.513591}};
        \node[unanode] (a) at (2, 0) {\argnode{a}{0.1}{0.081886}};

        \draw[-stealth, thick] (f) -- node[pos=0.5, above=1pt] {+} (e);
        \draw[-stealth, thick] (e) -- node[pos=0.5, left=1pt] {+} (b);
        \draw[-stealth, thick] (e) -- node[pos=0.5, above=1pt] {+} (d);
        \draw[-stealth, thick] (e) -- node[pos=0.5, right=1pt] {+} (c);
        \draw[-stealth, thick] (b) -- node[pos=0.5, right=1pt] {-} (a);
        \draw[-stealth, thick] (d) -- node[pos=0.5, above=1pt] {+} (a);
        \draw[-stealth, thick] (c) -- node[pos=0.5, left=1pt] {-} (a);
    \end{tikzpicture}
    \caption{$\sctrbsempty$ violates (quantitative) counterfactuality w.r.t. SD-DFQuAD semantics $\fs$.}
    \label{fig:proof-shapley-counterfactuality-SD-DFQuAD}
\end{figure}
\begin{counterexample}[(Quantitative) Counterfactuality; $\sctrbsempty$ w.r.t. EB]\label{ce:shapley-counterfactuality-EB}
    Consider the QBAG displayed in \autoref{fig:proof-shapley-counterfactuality-EB}, which we denote by $\graph = (\Args, \is, \Att, \Supp)$.
    Given EB semantics, we observe that $\sctrbs{\{f\}}{\arga} \approx  3.4380 \times 10^{-6}$. 
    According to the definitions of counterfactuality and quantitative counterfactuality we must have $\fs_{\graph}(a) > \fs_{\graph \downarrow_{\Args \setminus \{f\}}}(a)$, which does not hold. 
    This illustrates the violation of the counterfactuality and quantitative counterfactuality principles.
\end{counterexample}
\begin{figure}[!ht]
    \centering
    \begin{tikzpicture}[node distance=1.5cm]
        \node[unanode] (f) at (2, 2) {\argnode{f}{1.0}{1.0}};
        \node[unanode] (e) at (0, 2) {\argnode{e}{0.025}{0.064218}};
        \node[unanode] (d) at (0, 0) {\argnode{d}{0.54}{0.550455}};
        \node[unanode] (b) at (-2, 0) {\argnode{b}{0.11}{0.115812}};
        \node[unanode] (c) at (2, 0) {\argnode{c}{0.1}{0.1055394}};
        \node[unanode] (g) at (4, 0) {\argnode{g}{0.4}{0.4}};
        \node[unanode] (a) at (2, -2) {\argnode{a}{0.3}{0.288789}};

        \draw[-stealth, thick] (f) -- node[pos=0.5, above=1pt] {+} (e);
        \draw[-stealth, thick] (e) -- node[pos=0.5, left=1pt] {+} (b);
        \draw[-stealth, thick] (e) -- node[pos=0.5, left=1pt] {+} (d);
        \draw[-stealth, thick] (e) -- node[pos=0.5, right=1pt] {+} (c);
        \draw[-stealth, thick] (b) -- node[pos=0.5, right=1pt] {-} (a);
        \draw[-stealth, thick] (d) -- node[pos=0.5, above=1pt] {+} (a);
        \draw[-stealth, thick] (c) -- node[pos=0.5, right=1pt] {-} (a);
        \draw[-stealth, thick] (g) -- node[pos=0.5, right=1pt] {-} (a);
    \end{tikzpicture}
    \caption{$\sctrbsempty$ violates (quantitative) counterfactuality w.r.t. EB semantics $\fs$.}
    \label{fig:proof-shapley-counterfactuality-EB}
\end{figure}
\begin{counterexample}[(Quantitative) Counterfactuality; $\sctrbsempty$ w.r.t. EBT]\label{ce:shapley-counterfactuality-EBT}
    Consider the QBAG displayed in \autoref{fig:proof-shapley-counterfactuality-EBT}, which we denote by $\graph = (\Args, \is, \Att, \Supp)$.
    Given EBT semantics, we observe that $\sctrbs{\{f\}}{\arga} \approx  -2.7043 \times 10^{-5}$.
    According to the definitions of counterfactuality and quantitative counterfactuality we must have $\fs_{\graph}(a) < \fs_{\graph \downarrow_{\Args \setminus \{f\}}}(a)$, which does not hold. This illustrates the violation of the counterfactuality and quantitative counterfactuality principles.
\end{counterexample}
\begin{figure}[!ht]
    \centering
    \begin{tikzpicture}[node distance=1.5cm]
        \node[unanode] (f) at (-2, 2) {\argnode{f}{1.0}{1.0}};
        \node[unanode] (e) at (0, 2) {\argnode{e}{0.025}{0.14146}};
        \node[unanode] (d) at (0, 0) {\argnode{d}{0.51}{0.447405}};
        \node[unanode] (b) at (-2, 0) {\argnode{b}{0.4}{0.424965}};
        \node[unanode] (c) at (2, 2) {\argnode{c}{0.55}{0.55}};
        \node[unanode] (g) at (4, 0) {\argnode{g}{0.429}{0.429}};
        \node[unanode] (a) at (2, -2) {\argnode{a}{0.3}{0.302988}};

        \draw[-stealth, thick] (f) -- node[pos=0.5, above=1pt] {-} (e);
        \draw[-stealth, thick] (e) -- node[pos=0.5, left=1pt] {+} (b);
        \draw[-stealth, thick] (e) -- node[pos=0.5, left=1pt] {+} (d);
        \draw[-stealth, thick] (b) -- node[pos=0.5, below=1pt] {-} (a);
        \draw[-stealth, thick] (d) -- node[pos=0.5, above=1pt] {+} (a);
        \draw[-stealth, thick] (c) -- node[pos=0.5, right=1pt] {-} (d);
        \draw[-stealth, thick] (g) -- node[pos=0.5, left=1pt] {-} (a);
        \draw[-stealth, thick] (g) -- node[pos=0.5, above=1pt] {-} (d);
    \end{tikzpicture}
    \caption{$\sctrbsempty$ violates (quantitative) counterfactuality w.r.t. EBT semantics $\fs$.}
    \label{fig:proof-shapley-counterfactuality-EBT}
\end{figure}
%
\begin{counterexample}[(Quantitative) Counterfactuality; $\sctrbgmempty$ w.r.t. QE]\label{ce:gradient-counterfactuality-QE}
    Consider the QBAG displayed in \autoref{fig:proof-gradient-counterfactuality-QE}, which we denote by $\graph = (\Args, \is, \Att, \Supp)$.
    Given QE semantics, it holds that $\sctrbs{\{d\}}{\arga} =  0.0$.
    According to the definition of counterfactuality we must have $\fs_{\graph}(a) = \fs_{\graph \downarrow_{\Args \setminus \{d\}}}(a)$, which does not hold (we observe $\fs_{\graph}(a) < \fs_{\graph \downarrow_{\Args \setminus \{d\}}}(a)$), providing a violation of the counterfactuality and quantitative counterfactuality principles.
\end{counterexample}
\begin{figure}[!ht]
    \centering
    \begin{tikzpicture}[node distance=1.5cm]
        \node[unanode] (d) at (0, 1) {\argnode{d}{0.4}{0.4}};
        \node[unanode] (b) at (2, 1) {\argnode{b}{0.4}{0.4828}};
        \node[unanode] (c) at (-2, 1) {\argnode{c}{0.4}{0.4828}};
        \node[unanode] (a) at (0, -1) {\argnode{a}{0.2}{0.1980}};
        \node[unanode] (e) at (-2, -1) {\argnode{e}{0.1}{0.1}};

        \draw[-stealth, thick] (d) -- node[pos=0.5, above=1pt] {+} (b);
        \draw[-stealth, thick] (d) -- node[pos=0.5, above=1pt] {+} (c);
        \draw[-stealth, thick] (b) -- node[pos=0.5, above=1pt] {-} (a);
        \draw[-stealth, thick] (c) -- node[pos=0.5, above=1pt] {+} (a);
        \draw[-stealth, thick] (e) -- node[pos=0.5, above=1pt] {-} (a);
        
    \end{tikzpicture}
    \caption{$\sctrbgmempty$ violates (quantitative) counterfactuality w.r.t. QE semantics $\fs$.}
    \label{fig:proof-gradient-counterfactuality-QE}
\end{figure}
\begin{counterexample}[(Quantitative) Counterfactuality; $\sctrbgmempty$ w.r.t. DFQuAD]\label{ce:gradient-counterfactuality-DFQuAD}
    Consider the QBAG displayed in \autoref{fig:proof-gradient-counterfactuality-DFQuAD}, which we denote by $\graph = (\Args, \is, \Att, \Supp)$.
    Given DFQuAD semantics, we observe that $\sctrbgm{\{e\}}{\arga} \approx  7.4506 \times 10^{-9}$.
    According to the definitions of counterfactuality and quantitative counterfactuality we must have $\fs_{\graph}(a) > \fs_{\graph \downarrow_{\Args \setminus \{e\}}}(a)$, which does not hold. 
    This illustrates the violation of the counterfactuality and quantitative counterfactuality principles.
\end{counterexample}
\begin{figure}[!ht]
    \centering
    \begin{tikzpicture}[node distance=1.5cm]
        \node[unanode] (e) at (0, 2) {\argnode{c}{0.5}{0.5}};
        \node[unanode] (b) at (-2, 0) {\argnode{d}{0.0}{0.5}};
        \node[unanode] (d) at (0, 0) {\argnode{b}{0.0}{0.5}};
        \node[unanode] (c) at (2, 0) {\argnode{a}{0.0}{0.5}};
        \node[unanode] (a) at (0, -2) {\argnode{a}{0.5}{0.375}};
        
        \draw[-stealth, thick] (e) -- node[pos=0.5, above=1pt] {+} (b);
        \draw[-stealth, thick] (e) -- node[pos=0.5, right=1pt] {+} (d);
        \draw[-stealth, thick] (e) -- node[pos=0.5, above=1pt] {+} (c);
        \draw[-stealth, thick] (b) -- node[pos=0.5, above=1pt] {+} (a);
        \draw[-stealth, thick] (d) -- node[pos=0.5, right=1pt] {-} (a);
        \draw[-stealth, thick] (c) -- node[pos=0.5, above=1pt] {-} (a);
        
    \end{tikzpicture}
    \caption{$\sctrbgmempty$ violates (quantitative) counterfactuality w.r.t. DFQuAD semantics $\fs$.}
    \label{fig:proof-gradient-counterfactuality-DFQuAD}
\end{figure}
\begin{counterexample}[(Quantitative) Counterfactuality; $\sctrbgmempty$ w.r.t. SD-DFQuAD]\label{ce:gradient-counterfactuality-SD-DFQuAD}
    Consider the QBAG displayed in \autoref{fig:proof-gradient-counterfactuality-SD-DFQuAD}, which we denote by $\graph = (\Args, \is, \Att, \Supp)$.
    Given SD-DFQuAD semantics, we observe that $\sctrbgm{\{b\}}{\arga} = -0.25$.
    According to the definitions of counterfactuality and quantitative counterfactuality we must have $\fs_{\graph}(a) < \fs_{\graph \downarrow_{\Args \setminus \{b\}}}(a)$, which does not hold. 
    This illustrates the violation of the counterfactuality and quantitative counterfactuality principles.
\end{counterexample}
\begin{figure}[!ht]
    \centering
    \begin{tikzpicture}[node distance=1.5cm]
        \node[unanode] (c) at (-2, 0) {\argnode{c}{1.0}{1.0}};
        \node[unanode] (b) at (0, 0) {\argnode{b}{0.0}{0.0}};
        \node[unanode] (a) at (2, 0) {\argnode{a}{0.5}{0.5}};
        
        \draw[-stealth, thick] (c) -- node[pos=0.5, above=1pt] {-} (b);
        \draw[-stealth, thick] (b) -- node[pos=0.5, above=1pt] {-} (a);
    \end{tikzpicture}
    \caption{$\sctrbgmempty$ violates (quantitative) counterfactuality w.r.t. SD-DFQuAD semantics $\fs$.}
    \label{fig:proof-gradient-counterfactuality-SD-DFQuAD}
\end{figure}
\begin{counterexample}[(Quantitative) Counterfactuality; $\sctrbgmempty$ w.r.t. EB and EBT]\label{ce:gradient-counterfactuality-EB-EBT}
    Consider the QBAG displayed in \autoref{fig:proof-gradient-counterfactuality-EB-EBT}, which we denote by $\graph = (\Args, \is, \Att, \Supp)$.
    Given EB and EBT semantics, we observe that $\sctrbgm{\{b\}}{\arga} \approx -0.4530$.
    According to the definitions of counterfactuality and quantitative counterfactuality we must have $\fs_{\graph}(a) < \fs_{\graph \downarrow_{\Args \setminus \{b\}}}(a)$, which does not hold. 
    This illustrates the violation of the counterfactuality and quantitative counterfactuality principles.
\end{counterexample}
\begin{figure}[!ht]
    \centering
    \begin{tikzpicture}[node distance=1.5cm]
        \node[unanode] (c) at (-2, 0) {\argnode{c}{1.0}{1.0}};
        \node[unanode] (b) at (0, 0) {\argnode{d}{0.0}{0.0}};
        \node[unanode] (a) at (2, 0) {\argnode{b}{0.5}{0.5}};
        
        \draw[-stealth, thick] (c) -- node[pos=0.5, above=1pt] {+} (b);
        \draw[-stealth, thick] (b) -- node[pos=0.5, above=1pt] {-} (a);
    \end{tikzpicture}
    \caption{$\sctrbgmempty$ violates the (quantitative) counterfactuality principle w.r.t. EB and EBT semantics $\fs$.}
    \label{fig:proof-gradient-counterfactuality-EB-EBT}
\end{figure}
%
\end{document}